%% file: paper.tex
\documentclass[conference]{IEEEtran}
\IEEEoverridecommandlockouts
\usepackage[T1]{fontenc}

\usepackage{cite}
\usepackage{amsmath,amssymb,amsfonts}
\usepackage{graphicx}
\usepackage{textcomp}
\usepackage{xcolor}
\def\BibTeX{{\rm B\kern-.05em{\sc i\kern-.025em b}\kern-.08em
    T\kern-.1667em\lower.7ex\hbox{E}\kern-.125emX}}

\usepackage{algorithm}
\usepackage{algpseudocode}
\usepackage{amsthm}
\usepackage{subcaption}
\usepackage{url}
\usepackage{booktabs}
\usepackage{balance}
\usepackage{hyperref}

\RequirePackage{wasysym}
\newcommand{\yes}{\CIRCLE}
\newcommand{\no}{\Circle}

\newtheorem{theorem}{Theorem}
\newtheorem{definition}{Definition}
\newtheorem{corollary}{Corollary}
    
\begin{document}

\title{Differentially-Private Decision Trees and Provable Robustness to Data Poisoning}

% \author{\IEEEauthorblockN{Anynomous author(s)} \IEEEauthorblockA{Affiliation \\ Email}}

\author{\IEEEauthorblockN{Dani\"el Vos, Jelle Vos, Tianyu Li, Zekeriya Erkin and Sicco Verwer} \IEEEauthorblockA{Delft University of Technology\\
Email: d.a.vos@tudelft.nl, j.v.vos@tudelft.nl, tianyu.li@tudelft.nl, z.erkin@tudelft.nl, s.e.verwer@tudelft.nl}}

\maketitle

\begin{abstract}
Decision trees are interpretable models that are well-suited to non-linear learning problems. Much work has been done on extending decision tree learning algorithms with differential privacy, a system that guarantees the privacy of samples within the training data. However, current state-of-the-art algorithms for this purpose sacrifice much utility for a small privacy benefit. These solutions create random decision nodes that reduce decision tree accuracy or spend an excessive share of the privacy budget on labeling leaves. Moreover, many works do not support continuous features or leak information about them. We propose a new method called PrivaTree based on histograms that chooses good splits while consuming the privacy budget efficiently. The resulting trees provide a better privacy-utility trade-off and accept mixed numerical and categorical data without leaking information about numerical features. Finally, while it is notoriously hard to give robustness guarantees against data poisoning attacks, we demonstrate bounds for the expected accuracy and success rates of backdoor attacks against differentially-private learners. By leveraging the better privacy-utility trade-off of PrivaTree we are able to train decision trees with significantly better robustness against backdoor attacks compared to regular decision trees and with meaningful theoretical guarantees.
\end{abstract}


\input{sections/01_introduction}
\input{sections/preliminary}
\input{sections/02_related_work}
\input{sections/03_method}
\input{sections/04_poisoning_robustness}

\input{sections/05_results}
\input{sections/06_discussion}

\clearpage
\balance
\bibliographystyle{IEEEtran}
\bibliography{paper}

\clearpage

\appendix
\input{sections/appendix}

\end{document}

%% file: sections/01_introduction.tex
\section{Introduction}
Machine learning has achieved widespread success with neural networks and ensemble methods, but it is almost impossible for humans to understand the decisions such models make~\cite{DBLP:journals/queue/Lipton18}. Fortunately, much work has been done on training machine learning models that are directly interpretable by humans~\cite{rudin2019stop}. Especially size-limited decision trees~\cite{breiman1984cart,quinlan1986induction} are successful methods for their interpretability combined with their ability to predict non-linear data.

While decision trees can offer interpretability, they reveal information about the data they were trained on.
This is a detrimental property when models are trained on private data that contains sensitive information, such as in fraud detection and medical applications.
Differentially-private~\cite{DBLP:conf/icalp/Dwork06} machine learning models solve this problem by introducing carefully crafted randomness into the way the models are trained~\cite{DBLP:conf/ccs/AbadiCGMMT016}.
For differentially-private decision trees, the entire model consisting of decision node splits and leaf labels can be made public, and by extension, predictions made by the model. This is not only useful for training interpretable private models, but decision trees are also vital primitives for building ensembles~\cite{breiman2001random,friedman2002stochastic,chen2016xgboost,NIPS2017_6449f44a} for tabular data. The key problem in training differentially-private models is efficiently spending the privacy budget $\epsilon$ to achieve high utility. In this work, we propose an algorithm for training decision trees with an improved privacy-utility trade-off.

Many previous works have already proposed ways to generate differentially-private decision trees, but they also have shortcomings. There are two main categories of algorithms here. The first category~\cite{DBLP:journals/corr/BojarskiCCL14,DBLP:journals/eswa/FletcherI17,DBLP:journals/corr/abs-1907-02444,DBLP:journals/tdp/JagannathanPW12} chooses splits completely at random and allocates the entire privacy budget for labeling the leaves. The second category~\cite{DBLP:conf/pods/BlumDMN05,DBLP:phd/au/Borhan18,DBLP:conf/ausdm/Fletcher015,DBLP:conf/kdd/FriedmanS10} extends the greedy splitting procedure of regular decision trees, where splits are selected by optimizing a splitting criterion. These works guarantee differential privacy by incorporating noise into the splitting criterion while consuming a part of the user-defined privacy budget. However, since these approaches require computing many scores to select good decision nodes, an inefficient use of a mechanism results in consuming a large fraction of the privacy budget. The remaining budget is spent on labeling leaves.

In this work, we propose a method called PrivaTree to train differentially-private decision trees. PrivaTree uses the privacy budget more efficiently than previous work when choosing splits by leveraging histograms and using the permute-and-flip mechanism. We also propose a strategy for distributing the privacy budget that offers good performance for both very small and large datasets. The result is a practical method for training private trees with a significantly better utility for all privacy budgets. PrivaTree also prevents leakage from the location of splits on numerical features, a property that some previous methods do not have. Moreover, we demonstrate how the ability to train accurate differentially-private decision trees with small privacy budgets allows for performance guarantees against data poisoning attacks in which an adversary manipulates the training data. Our experiments on various tabular benchmark datasets demonstrate an improved privacy-utility trade-off compared to other private trees. We also experimentally demonstrate that PrivaTree offers stronger poisoning robustness guarantees than private logistic regression~\cite{chaudhuri2011differentially}, another interpretable method. Our experiments on the MNIST 0 vs 1 dataset show that indeed PrivaTrees with small privacy budgets resist a trigger-based backdoor attack.

\begin{figure*}[tb]
     \centering
     \begin{subfigure}[b]{0.495\textwidth}
         \centering
         \includegraphics[width=\textwidth]{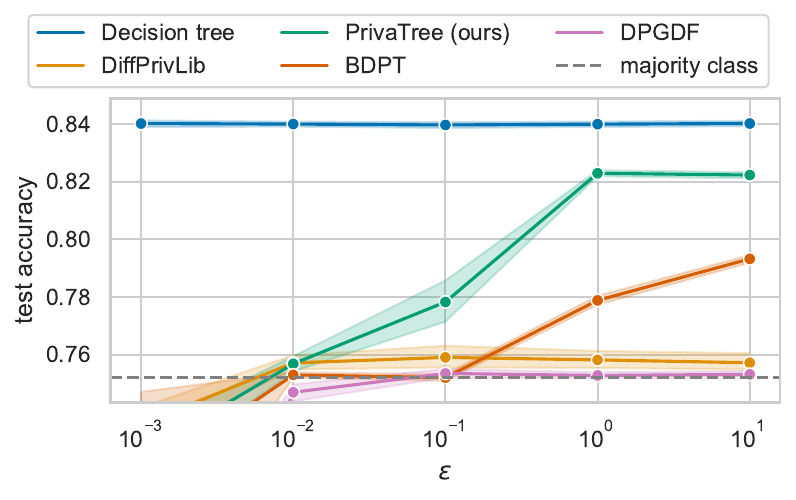}
         \caption{\textit{adult} dataset (45222 samples, 14 features)}
     \end{subfigure}
     \hfill
     \begin{subfigure}[b]{0.495\textwidth}
         \centering
         \includegraphics[width=\textwidth]{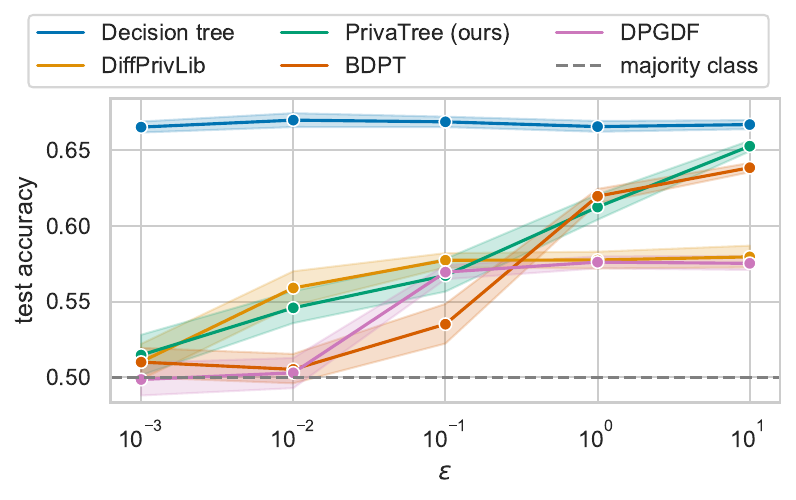}
         \caption{\textit{compas-two-years} (4966 samples, 11 features)}
         \label{fig:pol-epsilon-shares}
     \end{subfigure}
        \caption{Mean accuracy scores of depth 4 trees when varying privacy budget $\epsilon$ from private to less private with 50 repetitions. For extremely small privacy budgets DiffPrivLib performs best but for higher budgets, PrivaTree achieves a significantly higher accuracy.}
        \label{fig:varying-epsilon}
\end{figure*}

%% file: sections/preliminary.tex
\section{Preliminaries}
\subsection{Decision tree learning}
Decision trees are simple models that consist of nodes that apply logical rules to a sample, and leaves that hold a prediction. By simply following a path through the rules one ends up in a leaf of which the value is predicted. Decision trees have become a popular choice of model due to their straight-forward interpretation when limiting the size of the tree~\cite{DBLP:journals/queue/Lipton18} and their success in more complex ensembles~\cite{breiman2001random,friedman2002stochastic,chen2016xgboost,NIPS2017_6449f44a,grinsztajn2022tree}. While it is debatable what exact tree size maintains interpretability, we choose to train trees up to a depth of 4 resulting in at most 16 leaves.

The most popular algorithms for learning decision trees are based on CART~\cite{breiman1984cart} and ID3~\cite{quinlan1986induction}. For classification tasks, these algorithms recursively create decision nodes that minimize Gini impurity or maximize information gain and create leaves labeled with the majority of labels that reach them. While these methods are greedy heuristics and thus offer few guarantees~\cite{kearns1996boosting}, they perform well in practice. Due to the success of decision trees in gradient boosting, much effort has gone into implementing efficient algorithms using histograms~\cite{chen2016xgboost,NIPS2017_6449f44a}. We base PrivaTree on such histogram-based learners, but where the choice of histograms in boosting was motivated by runtime efficiency we leverage them to achieve a better trade-off between privacy and accuracy.

\subsection{Differential privacy}

Differential privacy \cite{DBLP:conf/icalp/Dwork06,DBLP:journals/fttcs/DworkR14,DBLP:conf/tcc/DworkMNS06} provides strong privacy guarantees for algorithms over aggregate datasets, which implies that the existence of any record in the dataset does not influence the output probability with factors $\epsilon$ and $\delta$. This property prevents membership attacks, attacks aimed at determining whether specific samples were included in the train set of a model, with a high probability if $\epsilon$ is chosen small enough.

\begin{definition}[Differential privacy]
    A randomized algorithm $\mathcal{M}$ satisfies $(\epsilon,\delta)$-differential privacy if for all neighboring datasets $\mathcal{D},\mathcal{D}^\prime \in \mathbb{N}^{|\mathcal{X}|}$ differing in one element, and any $\mathcal{S}\subseteq \mathrm{Range}(\mathcal{M})$,
\begin{equation}
    \Pr[\mathcal{M}(\mathcal{D}) \in \mathcal{S}] \leq e^{\epsilon} \Pr[\mathcal{M}(\mathcal{D}^\prime) \in \mathcal{S}] + \delta 
\end{equation}
    where $\mathbb{N}$ is the set of non-negative integers and $\mathcal{X}$ is the universe for all datasets. If $\delta$ is $0$, $\mathcal{M}$ satisfies $\epsilon$-differential privacy.
\end{definition}

Among differentially private mechanisms, the Laplace mechanism adds noise to a numerical output,
and the exponential mechanism returns a precise output among a group according to the utility scores for each element. Both mechanisms are widely used and are defined as follows.

\begin{definition}[Laplace mechanism \cite{DBLP:journals/fttcs/DworkR14}]
    A randomized algorithm $\mathcal{M}$ satisfies $\epsilon$-differential privacy over a real value query $f:\mathbb{N}^{|\mathcal{X}|} \rightarrow \mathbb{R}^k$ if
\begin{equation}
    \mathcal{M}(\mathcal{D},\ f,\ \epsilon) = f(\mathcal{D}) + (y_1 \ \dots \ y_k)\ ,  \ \ \ \ \ y_i \sim \mathrm{Lap}\left(\frac{\Delta f}{\epsilon}\right) 
\end{equation}
where $\mathrm{Lap}(b)$ is the Laplace distribution with scale $b$ that $\mathrm{Lap}(x\ |\ b) = \frac{1}{2b}\exp(-\frac{|x|}{b})$, and $\Delta f$ is the ${l}_1$-sensitivity that
\begin{equation}
    \Delta f = \max \limits_{X , X' \in \mathbb{N}^{|\mathcal{X}|},  ||X-X'||_1 \leq 1} ||f(X) - f(X')||_1 
\end{equation}
\end{definition}

With the Laplace mechanism, the Laplace noise is added to the accurate output of the query $f$ so that the output of $\mathcal{M}$ satisfies $\epsilon$-differential privacy over the query. Other mechanisms, such as the Gaussian mechanism and geometric mechanism~\cite{DBLP:journals/siamcomp/GhoshRS12}, achieve differential privacy in a similar way. The geometric mechanism is similar to the Laplace mechanism but works with integer values. An algorithm $\mathcal{M}(\mathcal{D},\ f,\ \epsilon)$ is $\epsilon$-differentially private when the noise follows the two-sided geometric distribution:
\begin{equation}
    \Pr[y_i=\delta]=\frac{1-\epsilon}{1+\epsilon}\epsilon^{|\delta|}
\end{equation}
with a query $f$, the parameter $\epsilon \in (0,1)$ and every integer $\delta$.

\begin{definition}[Exponential mechanism \cite{DBLP:conf/focs/McSherryT07}]
    A randomized algorithm $\mathcal{M}$ satisfies $\epsilon$-differential privacy over a utility function $u:\mathbb{N}^{|\mathcal{X}|}\times \mathcal{R}\rightarrow \mathbb{R}$ if $\mathcal{M}$ selects an element $t\in \mathcal{R}$ with the probability that
\begin{equation}
\begin{aligned}
     \Pr[\mathcal{M} & (\mathcal{D}, u, \epsilon, \mathcal{R}) = t \in \mathcal{R}] = \\
    &\frac{\exp{(\epsilon u(\mathcal{D},t)/(2\Delta u)}) \cdot \mu (t)}{\sum _{r\in\mathcal{R}} \exp{((\epsilon u(\mathcal{D},r)/(2\Delta u))} \cdot \mu (r)}
\end{aligned}
\end{equation}
    where $\Delta u$ is the sensitivity that 
    \begin{equation}
        \Delta u = \max \limits_{r \in \mathcal{R}} \max \limits_{X, X' \in \mathbb{N}^{|\mathcal{X}|}, ||X-X'||_1 \leq 1} |u(X,r) - u(X',r)|
    \end{equation}
    where $\mathcal{R}$ is the range for output, and $\mu$ is a measure over $\mathcal{R}$.
\end{definition}

In contrast to the Laplace mechanism, the exponential mechanism assigns a probability to each possible output in the group according to the utility score. By doing that, the mechanism can output a precise element using the probabilities, and perturbation is added for the selection procedure. Similarly, permute-and-flip~\cite{mckenna2020permute} randomly chooses a value from a set of options, weighed by a utility score and the privacy parameter $\epsilon$. For each possible output, the mechanism simulates flipping a biased coin, and the item is returned if the head is up with a probability according to an exponential function. Otherwise, it flips the coin for the next item with new probabilities assigned. Compared to the exponential mechanism, the probability of outputting an item is updated for each round, and the authors of~\cite{mckenna2020permute} also show that the permute-and-flip mechanism never performs worse than the exponential mechanism in expectation and better in other situations.

Meanwhile, differential privacy holds sequential and parallel composition properties \cite{DBLP:conf/sigmod/McSherry09} as shown in Theorems \ref{theorem:seq} and \ref{theorem:par}. For a series of mechanisms $\mathcal{M}_{[k]}=(\mathcal{M}_1,\dots,\mathcal{M}_k)$, each $\mathcal{M}_i$ is $\epsilon_i$-differentially private for $i \in [k]$. The sequential composition indicates that $(\sum _i\epsilon_i)$-differential privacy is guaranteed if the series of mechanisms are applied sequentially to the input, and the parallel composition implies that $(\max _i\epsilon_i)$-differential privacy is achieved if the series of mechanisms are applied to different disjoint subsets of the input.

\begin{theorem}[Sequential composition]
    If mechanism $\mathcal{M}_i$ is $\epsilon_i$-differentially private, the sequence of $\mathcal{M}_{[k]}(X)$ provides $(\sum _i\epsilon_i)$-differential privacy.
    \label{theorem:seq}
\end{theorem}

\begin{theorem}[Parallel composition]
    If mechanism $\mathcal{M}_i$ is $\epsilon_i$-differentially private, and $\mathcal{D}_i$ are disjoint subsets of the input domain $\mathcal{D}$,
    the sequence of $\mathcal{M}_{[k]}(X\cap \mathcal{D}_i)$ provides $(\max _i\epsilon_i)$-differential privacy.
        \label{theorem:par}
\end{theorem}

The privacy parameter $\epsilon$ is often referred to as the privacy budget as it can be intuitively interpreted this way. Following the intuition a private algorithm has a total budget of $\epsilon$ that it can spend on composed operations. Operations that get little budget will return noisier results than operations with much budget so by composing operations in parallel when possible and by carefully distributing the budget over the operations one can achieve a good trade-off between privacy and utility.

\subsection{Poisoning robustness}

Data poisoning attacks are attacks in which a malicious actor adds, removes, or modifies the data that a machine learning model is trained on. By making specific changes to the dataset, attackers can for example reduce the model's performance or plant a backdoor that allows the attacker to inject a specific 'trigger pattern' into an input of the model at test time to control its prediction. While much work has gone into defending against poisoning attacks \cite{jagielski2018manipulating,chen2018detecting,wong2018provable} and into provable robustness against other attacks such as evasion \cite{raghunathan2018certified,lecuyer2019certified} it is notoriously difficult to provide \textit{provable guarantees} on poisoning defenses. In this work, we demonstrate the provable poisoning robustness of differentially-private decision trees. We discuss existing methods for certified poisoning robustness in section \ref{sec:related-work-poisoning}.

%% file: sections/02_related_work.tex
\input{resources/table_related_work}

\section{Related work}

\subsection{Differentially-private decision tree learning}
Many previous works already propose algorithms for training differentially-private decision trees.
These algorithms address the privacy leakage in regular trees by replacing the node splitting and leaf labeling operations with differentially-private alternatives. This paper only considers algorithms for single decision trees. 
Fletcher and Islam~\cite{DBLP:journals/csur/FletcherI19} wrote a survey on this topic which also examines ensembles such as in private boosting~\cite{asadi2022private}.

Table~\ref{tab:related-work} provides a summary of existing algorithms for training private decision trees.
In the `features' column, we indicate whether the algorithm considers categorical and numerical features. We remark that algorithms for numerical splits can also support categories using one-hot encoding, and algorithms for categories can heuristically support numerical features by applying binning. Note that unless computed using differential privacy, the resulting bins reveal information about the training data. In that case, the model only guarantees differential privacy for the leaves, which is equivalent to labelDP~\cite{DBLP:conf/nips/GhaziGKMZ21}. The mechanism columns of the table indicate the private mechanisms used for node splitting and leaf labeling. Besides the specific choice of mechanism, the way these mechanisms are sequentially applied and the way the privacy budget is distributed have a significant effect on the algorithm's performance.

There are two main categories of algorithms for training differentially-private trees. The first category trains random decision trees, that replace split searching by splitting uniformly at random from the domain of possible feature values. A benefit of doing so is that splitting does not depend on the data and does not consume any privacy budget so labeling can be performed with the comeplete budget. However, random splits do not necessarily produce good leaves as having worse splits leads to leaves that contain a mix of samples from all classes. As a result, accurate labels will still cause misclassifications. For certain datasets, the poor quality of random splits strongly affects the performance of the resulting tree. For this reason, random decision trees are almost exclusively used in ensembles. Examples of such algorithms are Private-RDT~\cite{DBLP:journals/tdp/JagannathanPW12}, dpRFMV/dpRFTA~\cite{DBLP:journals/corr/BojarskiCCL14}, Smooth Random Trees~\cite{DBLP:journals/eswa/FletcherI17} and the implementation of DiffPrivLib trees~\cite{DBLP:journals/corr/abs-1907-02444}.

The second category consists of algorithms that train a greedy tree by probabilistically choosing a split weighed by a scoring function such as the information gain or the Gini impurity. SuLQ ID3~\cite{DBLP:conf/pods/BlumDMN05} and SuLQ-based ID3~\cite{DBLP:conf/kdd/FriedmanS10} do so by adding Gaussian or Laplace noise to the scores themselves, while works like DiffPID3~\cite{DBLP:conf/kdd/FriedmanS10}, DiffGen~\cite{DBLP:conf/kdd/MohammedCFY11}, DT-Diff~\cite{DBLP:conf/trustcom/ZhuXXZ13} and TrainSingleTree~\cite{DBLP:conf/aaai/LiWWH20} do so not by perturbing the scores, but using the exponential mechanism so the privacy budget does not have to be divided over so many queries. DPDF~\cite{DBLP:conf/ausdm/Fletcher015} further increases the utility of the queries by bounding the sensitivity of the Gini impurity, while ADiffP~\cite{DBLP:phd/au/Borhan18} dynamically allocates the privacy budget.

We compare against three of the latest algorithms for training private trees. BDPT~\cite{DBLP:journals/compsec/GuanSSWD20} is a greedy tree algorithm that uses the exponential mechanism for splitting but with smooth sensitivity, allowing for a higher utility per query as compared to previous works. DPGDF~\cite{DBLP:conf/icassp/XinY0H19} is a similar algorithm that uses smooth sensitivity for creating the leaves rather than the splits. This algorithm only support categorical features. Finally, DiffPrivLib~\cite{DBLP:journals/corr/abs-1907-02444} offers a recent implementation of random trees. For the leaves it uses the permute-and-flip mechanism, which performs better in practice than the exponential mechanism~\cite{mckenna2020permute}. This algorithm only supports numerical features.

\subsection{Provable poisoning robustness} \label{sec:related-work-poisoning}

While it is difficult to provide strong guarantees for poisoning robustness, several works have done so in the past. Steinhardt et al.~\cite{steinhardt2017certified} remove outliers based on metrics such as distance from the centroid of a class and can formally guarantee the accuracy of convex learners under specific poisoning attacks. Rosenfeld et al.~\cite{rosenfeld2020certified} consider label-flipping attacks in which only the labels can be changed by the attacker. They certify robust prediction of specific test samples for linear models by giving a bound on the minimum number of training labels that have to be flipped to change the test sample's prediction. Deep Partition Aggregation (DPA)~\cite{levine2021deep} trains an ensemble on disjoint subsets of the training dataset such that each sample is seen by only one member of the ensemble. Then, for a specific test sample, one can compute a bound on the number of train samples that should be perturbed to flip the prediction. These previous methods do not apply to single decision tree models.

Several works have guaranteed robustness using differential privacy. Ma et al.~\cite{ma2019data} prove general bounds on the attack cost for varying numbers of poisoned samples and apply this analysis to logistic regression. Also, several works have used differentially-private neural networks to be more robust against poisoning attacks~\cite{hong2020effectiveness,geiping2020witches,borgnia2021dp}, sometimes even demonstrating robustness in configurations in which the theoretical guarantees do not apply. Compared to DPA, differential privacy guarantees robustness of a global metric such as accuracy while methods like DPA guarantee robust predictions for each test sample individually. The link between robustness and privacy is also explored in several works~\cite{dwork2009differential,georgiev2022privacy,hopkins2023robustness}.

While multiple works have considered the poisoning robustness of differentially-private learners, to the best of our knowledge there is no demonstration of strong robustness guarantees for single decision trees. Drews et al.~\cite{drews2020proving} do analyze the robustness of regular non-private decision tree learners but show that for trees of depth 4 only 4 samples are often enough to reduce the accuracy guarantees to 0. As we will demonstrate, private decision trees with a good privacy-utility trade-off can guarantee poisoning robustness and outperform logistic regression in this regard. For the sake of completeness, we also compare against DPA~\cite{levine2021deep} which uses decision trees inside its ensemble to offer strong guarantees, but sacrifices interpretability and privacy.

%% file: resources/table_related_work.tex
\begingroup
\begin{table}[tb]

\setlength{\tabcolsep}{5pt}

\centering
\caption{Overview of methods for training differentially private decision trees, algorithms marked with * use smooth sensitivity. Most methods use the exponential mechanism $\mathcal{M}_{EM}$ for splitting and $\mathcal{M}_\text{Laplace}$ for labeling leaves. Methods without splitting a mechanism use random trees.}
\label{tab:related-work}

\begin{tabular}{@{}lcc|cc|cc@{}}
\toprule
\multicolumn{3}{c|}{\textbf{Method}} & \multicolumn{2}{c|}{\textbf{Features}} & \multicolumn{2}{c}{\textbf{Mechanism}} \\
Name & Year & Ref & Cat. & Num. & Splitting & Labeling \\ \midrule
SuLQ ID3 & 2005 & \cite{DBLP:conf/pods/BlumDMN05} & \yes & \no & $\mathcal{M}_\text{Gaussian}$ & $\mathcal{M}_\text{Gaussian}$ \\
Private-RDT & 2009 & \cite{DBLP:journals/tdp/JagannathanPW12} & \yes & \yes & - & $\mathcal{M}_\text{Laplace}$ \\
SuLQ-based ID3 & 2010 & \cite{DBLP:conf/kdd/FriedmanS10} & \yes & \yes & $\mathcal{M}_\text{Laplace}$ & $\mathcal{M}_\text{Laplace}$ \\
DiffPID3 & 2010 & \cite{DBLP:conf/kdd/FriedmanS10} & \yes & \yes & $\mathcal{M}_{EM}$ & $\mathcal{M}_\text{Laplace}$ \\
DiffGen & 2011 & \cite{DBLP:conf/kdd/MohammedCFY11} & \yes & \yes & $\mathcal{M}_{EM}$ & $\mathcal{M}_\text{Laplace}$ \\
DT-Diff & 2013 & \cite{DBLP:conf/trustcom/ZhuXXZ13} & \yes & \yes & $\mathcal{M}_{EM}$ & $\mathcal{M}_\text{Laplace}$ \\
dpRFMV/dpRFTA & 2014 & \cite{DBLP:journals/corr/BojarskiCCL14} & \yes & \yes & - & $\mathcal{M}_\text{Laplace}$ \\
DPDF & 2015 & \cite{DBLP:conf/ausdm/Fletcher015} & \yes & \no & $\mathcal{M}_{EM}$ & $\mathcal{M}_\text{Laplace}$ \\
Rana et al. & 2015 & \cite{DBLP:conf/icdm/RanaGV15} & \yes & \yes & $\mathcal{M}_\text{Laplace}$ & $\mathcal{M}_\text{Laplace}$ \\
Smooth Random & 2017 & \cite{DBLP:journals/eswa/FletcherI17} & \yes & \yes & - & $\mathcal{M}_{EM}$* \\
ADiffP & 2018 & \cite{DBLP:phd/au/Borhan18} & \yes & \no & $\mathcal{M}_{EM}$ & $\mathcal{M}_\text{Laplace}$ \\
DPGDF & 2019 & \cite{DBLP:conf/icassp/XinY0H19} & \yes & \no & $\mathcal{M}_{EM}$ & $\mathcal{M}_{EM}$* \\
BDPT & 2020 & \cite{DBLP:journals/compsec/GuanSSWD20} & \yes & \yes & $\mathcal{M}_{EM}$* & $\mathcal{M}_\text{Laplace}$ \\
TrainSingleTree & 2020 & \cite{DBLP:conf/aaai/LiWWH20} & \no & \yes & $\mathcal{M}_{EM}$ & $\mathcal{M}_\text{Laplace}$ \\ 
DiffPrivLib & 2021 & \cite{DBLP:journals/corr/abs-1907-02444} & \no & \yes & - & $\mathcal{M}_{PF}$ \\
PrivaTree & \multicolumn{2}{c|}{\emph{This work}} & \yes & \yes & $\mathcal{M}_\text{Geometric}$ & $\mathcal{M}_{PF}$ \\
\bottomrule
\end{tabular}
\end{table}
\endgroup

%% file: sections/03_method.tex
\begin{figure*}[tb]
     \centering
     \begin{subfigure}[b]{0.45\textwidth}
         \centering
         \includegraphics[width=\textwidth]{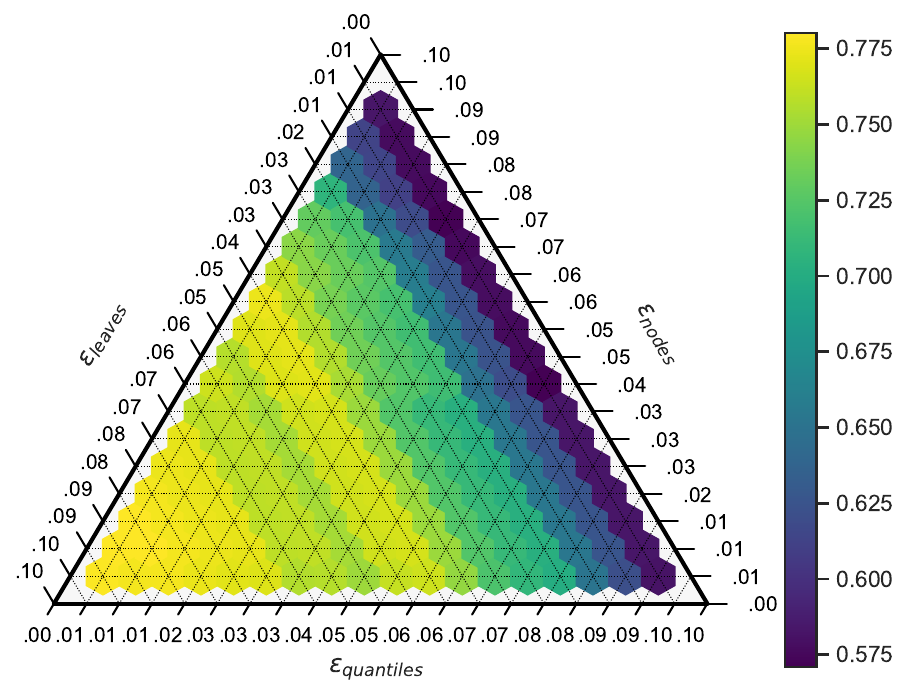}
         \caption{\textit{vote} dataset (435 samples, 16 features)}
     \end{subfigure}
     \hfill
     \begin{subfigure}[b]{0.45\textwidth}
         \centering
         \includegraphics[width=\textwidth]{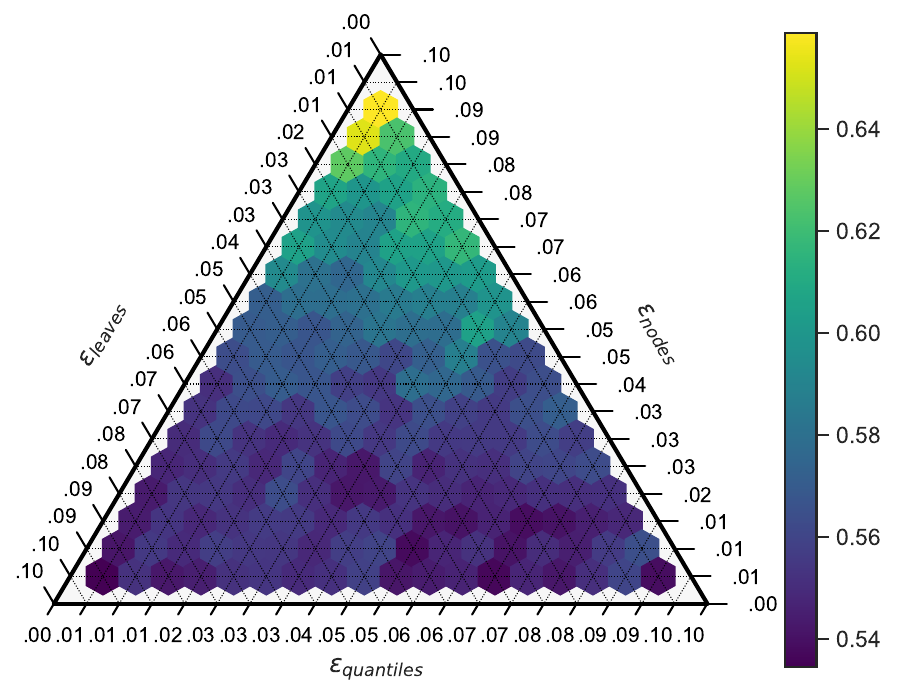}
         \caption{\textit{pol} dataset (10082 samples, 26 features)}
     \end{subfigure}
        \caption{Training accuracy scores when varying the distribution of the privacy budget $\epsilon = 0.1$ over different parts of the PrivaTree algorithm with depth 3. Scores were averaged over 50 executions. On small datasets such as \textit{vote}, the best trees allocate more budget to the leaves, while large datasets such as \textit{pol} benefit from more budget for splitting.
        }
        \label{fig:epsilon-shares}
\end{figure*}

\section{Improving differentially private decision trees: PrivaTree}
In this section, we present PrivaTree, which is an algorithm for training differentially-private decision trees with high utility.
PrivaTree incorporates three techniques to improve performance:
\begin{itemize}
    \item Histograms to limit the number of splits to consider for each decision node.
    \item A better distribution of privacy budget based on a bound for labeling leaves accurately.
    \item The permute-and-flip mechanism instead of the exponential mechanism for leaf labeling and node splitting.
\end{itemize}
Additionally, PrivaTree pre-processes numerical features into bins which enables it to train on both numerical and categorical features. PrivaTree assumes that the range of numerical features, the set of categorical values, the set of class labels, and the dataset size are publicly known. In our experiments we compute these values based on the complete datasets. We provide 
pseudocode in Algorithm~\ref{alg:private-tree}, which we describe in more detail in the rest of this section.

\subsection{Scoring splits using histograms}
Classification tree learning algorithms such as CART create decision nodes by recursively choosing a split among all feature values that minimizes the weighted Gini impurity. Previous differentially-private decision trees used similar approaches but often leak information about numerical feature values. Friedman and Schuster~\cite{DBLP:conf/kdd/FriedmanS10} show how to privately split on numerical features using the continuous exponential mechanism but this requires more privacy budget because each feature needs to be considered separately. Another work, BDPT, permits some leakage, providing a weaker form of privacy by averaging every 5 numerical feature values, and then splitting using the exponential mechanism. 

\paragraph{Numerical features} To find high-quality splits while protecting feature value information and efficiently using the privacy budget, we use histograms. Splitting according to histograms has been a successful progression in decision tree learning for gradient boosting ensembles~\cite{chen2016xgboost,NIPS2017_6449f44a} where it is used for its runtime efficiency. Instead, we rely on them because they reduce the number of threshold values considered for splitting which improves node selection at limited privacy budget. Specifically, we compute the Gini impurity for every threshold value between the bins of all features and use the permute-and-flip mechanism with to choose the one that privately minimizes impurity. This results in decision node rules of the kind `feature value $\leq$ threshold'. In the rest of this paper, we fix the number of bins for numerical features to 10. The global sensitivity of this operation is $\frac{1}{2}$ as this is the range of the Gini impurity, but the performance can be improved using smooth sensitivity~\cite{DBLP:journals/compsec/GuanSSWD20}. Therefore we first privately count the number of samples that reach the node using the Geometric mechanism and budget $\frac{1}{2} \epsilon_\text{node}$ and then use the smooth sensitivity $1 {-} (\frac{n}{n + 1})^2 {-} (\frac{1}{n + 1})^2$ for the permute-and-flip mechanism with budget $\frac{1}{2} \epsilon_\text{node}$.

\paragraph{Categorical features}
Previous decision tree learning algorithms support categorical features by only splitting one category at a time. While this method is sound, it often requires deep trees to make enough splits for categorical features to be useful, and this harms interpretability. In the PrivaTree algorithm, we find a partition of the categories instead. Such rules are of the kind `categories $L$ go left, categories $R$ go right'. This is done by fixing the order of the categories and finding a threshold somewhere in this sequence of categories similar to finding a split on numerical data.

\subsection{Pre-processing by binning}
Before the splitting procedure, PrivaTree needs to select the boundaries of the 10 bins of the histograms for each numerical feature.
A natural choice is to create equal-width bins of numerical features based on the public knowledge of each feature's range, but there can be a problem with this approach: features with long tails would cause data to be concentrated in a few bins and would result in a large loss of information. Instead, we also consider binning numerical features using quantiles so that they evenly divide the samples over each bin.
Since regular quantiles leak information about the dataset, we also implement differentially-private quantiles. We use the \textit{jointexp} algorithm~\cite{gillenwater2021differentially} for this, which improves performance over optimal statistical estimation techniques~\cite{smith2011privacy} when computing multiple quantiles on the same data. This algorithm runs with privacy budget $\epsilon_\text{quantile}$ which is divided over each feature using sequential composition as we will later define.
Categorical variables require no pre-processing. Instead, we encode them as integers and the split-finding operation handles these values natively. In our results we compare quantiles with equal-width bins and suggest equal-width bins as a good baseline as computing it does not consume privacy budget.

\subsection{Leaf labeling according to majority votes}
Once the PrivaTree algorithm has produced a series of decision nodes by recursively splitting and reaches the stopping criterion, the algorithm creates a leaf containing a prediction. In the non-private setting, the prediction is usually chosen to be the majority of the class labels of samples that reach that leaf to maximize accuracy. However, this leaks private information. Previous works have used the Laplace or (smooth) exponential mechanisms but like modern implementations such as DiffPrivLib~\cite{DBLP:journals/corr/abs-1907-02444} we opt for the permute-and-flip mechanism~\cite{mckenna2020permute}. 
Permute-and-flip is proven to have an expected error that is at least as good or better than that of the exponential mechanism and practically outperforms it. To label a leaf we therefore count the number of samples of each class and apply permute-and-flip with privacy budget $\epsilon_\text{leaf}$ which we define later.

\begin{algorithm*}[tb]
    \caption{Train PrivaTree with $\epsilon$ differential privacy}\label{alg:private-tree}
    \textbf{Input:} dataset $X$ ($n \times m$), labels $y$, privacy budget $\epsilon$, maximum leaf error $\mathcal{E}_{max}$, maximum depth $d$
    \begin{algorithmic}[1]
        \State $\epsilon_\text{leaf} = \min(\frac{\epsilon}{2}, \epsilon'_\text{leaf})$, where $\epsilon'_\text{leaf}$ is computed with Equation \ref{eq:eps-leaf}
        \State $\epsilon_\text{node} = (\epsilon - \epsilon_\text{leaf})\cdot \frac{1}{d}$
        \For{numerical features $j_\text{num}$}
            \State bin values $X_{i,j_\text{num}}$ into bins $B_{j_\text{num}}$ of equal width (determined with public feature range)
        \EndFor
        \Procedure{FitTree}{$X', y', d'$}
            \If{$d'$ = $0$}
                \State compute class counts $N_0, N_1, ..., N_K$ for all $K$ classes
                \State \textbf{return} $\text{Leaf}(\mathcal{M}_{PF}(\langle N_0, N_1, ..., N_K \rangle))$ \Comment{with budget $\epsilon_\text{leaf}$ and sensitivity $1$}
            \Else
                \State create histograms $\forall j, k, b {\in} B_j : H_{j, b, k} \gets \sum_i [X'_{i,j}{=}b \land y'_i{=}k]$
                \State compute utilities with the negative Gini impurity $u_{j, b} \gets - \textsc{Gini}(H_{j, b, 0}, H_{j, b, 1}, ...)$
                \State count the number of samples that reach this node $n = \mathcal{M}_{Geom}(|X'|)$ \Comment{with budget $\frac{1}{2}\epsilon_\text{node}$ and sensitivity $1$}
                \State choose split $(j^*, b^*)$ with $\mathcal{M}_{PF}(\{ u_{j, b} : \forall j, k \})$ \Comment{with budget $\frac{1}{2}\epsilon_\text{node}$ and sensitivity $1 {-} (\frac{n}{n + 1})^2 {-} (\frac{1}{n + 1})^2$}
                \State partition $X'$ into $(X^\text{left}, X^\text{right})$, and $y'$ into $(y^\text{left}, y^\text{right})$ according to $(j^*, b^*)$
                \State $\mathcal{T}_\text{left} \gets$ \textsc{FitTree}($X^{\text{left}}, y^{\text{left}}, d' {-} 1$), \quad $\mathcal{T}_\text{right} \gets$ \textsc{FitTree}($X^{\text{right}}, y^{\text{right}}, d' {-} 1$)
                \State \textbf{return} Node$(j^*, b^*, \mathcal{T}_\text{left}, \mathcal{T}_\text{right})$
            \EndIf
        \EndProcedure
        \State \textbf{return} \textsc{FitTree}($X, y, d$)
    \end{algorithmic}
\end{algorithm*}

% TODO: explain smooth sensitivity

\subsection{Distributing the privacy budget}
The composability property of differential privacy allows modular algorithm design by breaking up the algorithm into differentially-private primitives. However, it is generally not obvious how to distribute the privacy budget $\epsilon$ over the primitives to maximize the expected utility of the outcomes. Previous works have tried various heuristics such as distributing $\epsilon$ equally over each private operation or distributing epsilon 50-50 between node and leaf operations.
To understand the role of budget distribution in private tree learning we visualized the average training accuracy over 50 runs when varying budget spent on parts of the algorithm (quantile finding, node splitting and leaf labeling) in Figure \ref{fig:epsilon-shares}.
We find that when $\epsilon$ is large compared to the dataset size, it is best to spend nearly all budget on node operations. When $\epsilon$ is small compared to the dataset size it is favorable to spend budget on leaf labeling, i.e. to create nearly random nodes and label the leaves accurately. Computing quantiles usually required too much budget to be useful so by default we decide to split features into equal-width bins according to the feature range (which is assumed to be public knowledge). For example if `age' is a feature with values between 0 and 120 years, splitting it into 10 bins results in bins: 0-12, 12-24, ... 106-120.
By noticing that we can bound the expected error incurred by labeling leaves for a given privacy budget, we propose a budget distribution scheme that scales well for varying values of $\epsilon$. When the privacy budget is low compared to the dataset size, set $\epsilon_\text{leaf} = \frac{\epsilon}{2}$, i.e. half of the budget. When the budget is relatively high, set $\epsilon_\text{leaf}$ such that the maximum expected error incurred by labeling $\mathbb{E}[\mathcal{E}(\mathcal{M}, \vec{N})]$ is at most equal to the user-specified labeling error limit $\mathcal{E}_{max}$. Distribute the remaining budget $\epsilon - \epsilon_\text{leaf}$ over node operations to improve algorithm utility for higher values of $\epsilon$. For all our results we used  $\mathcal{E}_{max} = 0.01$, i.e. if possible only allow an extra accuracy loss of 1\% due to leaf labeling errors.

\begin{corollary} \label{cor:max-leaf-error}
For $K$ classes, $n$ samples and depth $d$ trees, the amount of privacy budget $\epsilon'_\text{leaf}$ needed for labeling leaves with the permute-and-flip mechanism with expected error $\mathbb{E}[\mathcal{E}(\mathcal{M}_{PF}, \vec{N})]$ of at most $\mathcal{E}_{max}$ is:
\begin{equation} \label{eq:eps-leaf}
    \epsilon'_\mathrm{leaf} \leq \frac{2^d \max_p 2 \log(\frac{1}{p}) \left( 1 - \frac{1 - (1 - p)^K}{K p} \right)}{n \; \mathcal{E}_{max}}\;,
\end{equation}
\end{corollary}
\begin{proof}
This result naturally follows from applying the worst-case error bound proved for $\mathcal{M}_{PF}$ in the permute-and-flip paper~\cite{mckenna2020permute}. A complete proof is given in the appendix.
\end{proof}

Following the definition of $\epsilon'_\mathrm{leaf}$ this results in the privacy budget being distributed as
\begin{align} \label{eq:epsilon-definitions}
    \epsilon_\text{leaf} &= \min\left(\frac{\epsilon}{2}, \epsilon'_\text{leaf}\right), \\ 
    \epsilon_\text{node} &= \frac{\epsilon - \epsilon_\text{leaf}}{d}.
\end{align}

Next, we show that this budget distribution scheme actually provides differential privacy by showing that the composition of mechanisms spends privacy budget equal to $\epsilon$.
\begin{theorem}
    PrivaTree, as described in Algorithm \ref{alg:private-tree}, provides $\epsilon$-differential privacy.
\end{theorem}
\begin{proof}
    For numerical attributes in the training set $X$, PrivaTree first computes bins using public knowledge which does not consume privacy budget. After that, the algorithm recursively splits the root node and since each split creates two distinct partitions of the data, each data point is only used in $d$ nodes where $d$ is the maximum depth of the tree. The amount of budget spent on node splitting is then $d \cdot \epsilon_\text{node}$ through sequential composition. Finally, leaf labeling consumes $\epsilon_\text{leaf}$ of the privacy budget and since each data point is only used in a single leaf this follows parallel composition (spending $\epsilon_\text{leaf}$ once). Combining all the operations we find that the amount of privacy budget spent is:
    \begin{align*}
        \epsilon &= d\cdot \epsilon_\text{node} + \epsilon_\text{leaf} \\
        &= d \cdot \frac{\epsilon - \epsilon_\text{leaf}}{d} + \epsilon_\text{leaf} \\
        &= \epsilon - \epsilon_\text{leaf} + \epsilon_\text{leaf} \\
        &= \epsilon.
    \end{align*}

    Therefore Algorithm \ref{alg:private-tree} provides $\epsilon$-differential privacy. 
\end{proof}

% TODO: clarify private quantile experiment?

%% file: sections/04_poisoning_robustness.tex
\section{Poisoning Robustness} \label{sec:poisoning-robustness}

When using machine learning trained on crowd-sourced data, such as in federated learning scenarios, or on potentially manipulated data one has to consider malicious user behavior. One such threat is data poisoning, in which users insert $x$ data points into the training dataset to confuse the classifier or introduce a backdoor. Many defenses have been proposed such as using learning behavior to ignore backdoor data~\cite{li2021anti} or post-processing based on adversarial robustness to remove backdoors~\cite{wu2021adversarial} but such methods work heuristically and offer no guarantees. We use the fact that $\epsilon$-differentially-private machine learning algorithms are limited in sensitivity to dataset changes to provide guarantees against poisoning attacks. Ma et al.~\cite{ma2019data} introduce the attack cost $J(D) = \mathbb{E}[C(\mathcal{M}(D))]$, i.e. the expected value of a cost function $C$ that an attacker gets from models produced by $\mathcal{M}$ on dataset $D$. They prove the following about the cost of attacks against differentially-private learners:

\begin{theorem} \label{eq:poison-guarantee}
    \cite{ma2019data} Let $\mathcal{M}$ be an $\epsilon$-differentially-private learner. Let $J(\tilde{D})$ be the attack cost, where $\tilde{D} \ominus D \leq x$ (i.e. a dataset with $x$ poisoned samples compared to clean dataset $D$), then
    \begin{align}
        J(\tilde{D}) &\geq e^{-x \epsilon} J(D), \quad (C \geq 0) \label{eq:guarantee-nonnegative} \\
        J(\tilde{D}) &\geq e^{x \epsilon} J(D). \quad (C \leq 0) \label{eq:guarantee-nonpositive}
    \end{align}
\end{theorem}

We will consider two scenarios with different attacker objectives and their associated cost functions:
\begin{itemize}
    \item An attacker adds or removes up to $x$ samples from a dataset in an attempt to maximally reduce the accuracy of a model learned from the dataset. The associated cost function is the accuracy of the model, i.e. $J_\text{Acc}(\tilde{D}) = \mathbb{E}[\text{Accuracy}(\mathcal{M}(\tilde{D}))]$.
    \item An attacker adds $x$ copies of data points with a trigger pattern inserted into them and the associated label is changed to a target label in an effort to create a backdoor. The associated cost function is the Attack Success Rate (ASR), i.e. the percentage of data points for which the predicted label can be flipped to the target label. Since Theorem \ref{eq:poison-guarantee} assumes (without loss of generality) that the attacker minimizes $J$, we use the complement of the ASR as the cost, as the attacker wants to maximize ASR: $J_\text{ASR}(\tilde{D}) = \mathbb{E}[1 - \text{ASR}(\mathcal{M}(\tilde{D}))]$.
\end{itemize}

These two settings lead to the following robustness guarantees that we empirically evaluate in Section \ref{sec:results} for various private learners and datasets.

\begin{corollary} \label{cor:accuracy-guarantee}
    Let $\mathcal{M}$ be an $\epsilon$-differentially-private learner. Let $J_\text{Acc}(\tilde{D}) = \mathbb{E}[\text{Accuracy}(\mathcal{M}(\tilde{D}))]$ be the model's expected accuracy on the poisoned dataset $\tilde{D}$, where $\tilde{D} \in B(D, x)$ (i.e. a dataset with $x$ poisoned samples), then
    \begin{equation*}
        \mathbb{E}[\text{Accuracy}(\mathcal{M}(\tilde{D}))] \geq e^{-x \epsilon} \; \mathbb{E}[\text{Accuracy}(\mathcal{M}(D))]
    \end{equation*}
\end{corollary}
\begin{proof}
    Since the accuracy score function is non-negative this follows directly from Equation \ref{eq:guarantee-nonnegative}:
    \begin{align*}
        J(\tilde{D}) &\geq e^{-x \epsilon} J(D) \\
        \mathbb{E}[\text{Accuracy}(\mathcal{M}(\tilde{D}))] &\geq e^{-x \epsilon} \; \mathbb{E}[\text{Accuracy}(\mathcal{M}(D))] \;. \qedhere
    \end{align*}
\end{proof}

\begin{corollary} \label{cor:asr-guarantee}
    Let $\mathcal{M}$ be an $\epsilon$-differentially-private learner. Let $J_\text{ASR}(\tilde{D}) = \mathbb{E}[1 - \text{ASR}(\mathcal{M}(\tilde{D}))]$ be the expected backdoor attack success rate on the poisoned dataset $\tilde{D}$, where $\tilde{D} \in B(D, x)$ (i.e. a dataset with $x$ poisoned samples), then
    \begin{equation*}
        \mathbb{E}[\text{ASR}(\mathcal{M}(\tilde{D}))] \leq 1 - e^{- x \epsilon} \; \mathbb{E}[1 - \text{ASR}(\mathcal{M}(D))]
    \end{equation*}
\end{corollary}
\begin{proof}
    Since $J_\text{ASR}$ is non-negative we derive the guarantee from Equation \ref{eq:guarantee-nonnegative}:
    \begin{align*}
        J(\tilde{D}) &\geq e^{-x \epsilon} J(D) \\
        \mathbb{E}[1 - \text{ASR}(\mathcal{M}(\tilde{D}))] &\geq e^{-x \epsilon} \; \mathbb{E}[1 - \text{ASR}(\mathcal{M}(D))] \\
        \mathbb{E}[\text{ASR}(\mathcal{M}(\tilde{D}))] - \mathbb{E}[1] &\leq - e^{-x \epsilon} \; \mathbb{E}[1 - \text{ASR}(\mathcal{M}(D))] \\
        \mathbb{E}[\text{ASR}(\mathcal{M}(\tilde{D}))] &\leq 1 - e^{-x \epsilon} \; \mathbb{E}[1 - \text{ASR}(\mathcal{M}(D))] \;. \qedhere
    \end{align*}
\end{proof}

%% file: sections/05_results.tex
\input{resources/table_private_tree_results}

\input{resources/table_robustness_guarantees}

\section{Results} \label{sec:results}
We compare the performance of PrivaTree with regular decision trees from Scikit-learn~\cite{pedregosa2011scikit} and 3 previous methods for training private decision trees: DiffPrivLib~\cite{DBLP:journals/corr/abs-1907-02444}, BDPT~\cite{DBLP:journals/compsec/GuanSSWD20} and DPGDF~\cite{DBLP:conf/icassp/XinY0H19}. DiffPrivLib is a widely used Python library for differential privacy and implements several private machine learning models. Their decision tree implementation creates random decision nodes and uses all privacy budget to label leaves using the permute-and-flip mechanism. Since DiffPrivLib and Scikit-learn do not natively support categorical features we encode these into integers. BDPT and DPGDF did not share their implementations so we implemented these using primitives from DiffPrivLib. Since DPGDF only supports categorical variables we run experiments as in the work of Borhan~\cite{DBLP:phd/au/Borhan18} and remove numerical features. 
BDPT only heuristically protects numerical feature values so we compare it against PrivaTree*: a variant of PrivaTree where we compute 10 bins by computing quantiles non-privately. These quantiles still heuristically protect privacy about individual feature values of the dataset and the experiment allows us to compare the performance benefits of using quantiles instead of equal-width bins (PrivaTree).

We also compare the performance of PrivaTree on poisoning robustness with differentially-private logistic regression and Deep Partition Aggregation (DPA). Previous works on poisoning robustness with differential privacy have used private logistic regression~\cite{chaudhuri2011differentially} as a robust interpretable model. We use the implementation by DiffPrivLib~\cite{DBLP:journals/corr/abs-1907-02444} which perturbs the optimal model parameters. Before training, we center the data and normalize it to unit variance as this improves optimization. DPA trades interpretability and privacy for strong poisoning robustness guarantees by ensembling many models trained on distinct data subsets. We implement DPA with 1000 decision trees as tree-based models generally work better than neural networks for tabular data~\cite{grinsztajn2022tree}.

All experiments ran on a computer with 16GB of RAM and a 2 GHz Intel i5 processor with 4 cores. 
See our code for more details\footnote{\url{https://github.com/tudelft-cda-lab/PrivaTree}}.

\subsection{Predictive performance}
To compare PrivaTree to existing works we evaluated performance on two well-known benchmarks. First, we experimented on 6 datasets from the UCI repository~\cite{uci_repository} that previous works tested on. However, these datasets are often small (\textit{diabetes}), too easy to predict (\textit{nursery}), or imbalanced (\textit{adult}), which skews performance numbers. To complement this, we therefore also run experiments on the tabular data benchmark~\cite{grinsztajn2022tree}. These datasets were chosen to be real-world, balanced, not too small, and not too simple. We removed rows with missing values and computed the public feature ranges based on the datasets, the categorical values are supplied by OpenML~\cite{OpenML2013}.

In Table \ref{tab:performance-comparison} we present the accuracy scores on both benchmarks for trees of depth 4 computed with 5-fold stratified cross-validation at a privacy budget of $\epsilon = 0.1$. Results for other budgets are given in the appendix. Since all private algorithms are based on the standard greedy algorithm for decision trees, the goal is to score similarly to those non-private trees. On almost all datasets PrivaTree outperforms BDPT, DPGDF, and DiffPrivLib or performs similarly. On \textit{breast-w} and \textit{vote} however, DiffPrivLib sometimes performs better. This is because, on such small datasets, it is better to avoid spending the privacy budget on good decision nodes and instead spend all budget on labeling leaves correctly. On average, there is a 4\% difference in accuracy score between PrivaTree using equal-width bins and PrivaTree* using non-private quantiles. BDPT has previously only been tested on numerical features with few unique values and fails to train accurate trees on the numerical tabular benchmark. In Figure \ref{fig:varying-epsilon} we visualize the average accuracy over 50 runs when varying the total privacy budget $\epsilon$ for depth 4 trees on the \textit{adult} and \textit{compas-two-years} datasets. Again, when $\epsilon$ is small compared to dataset size DiffPrivLib outperforms the other methods. However, when is enough privacy budget to see value from choosing better decision nodes PrivaTree improves over the other methods. On average PrivaTree beats DiffPrivLib as when it wins the score is significantly higher while the score is only slightly lower when PrivaTree loses.

\begin{figure*}[tb]
     \centering
     \begin{subfigure}[b]{0.48\textwidth}
         \centering
         \includegraphics[width=\textwidth]{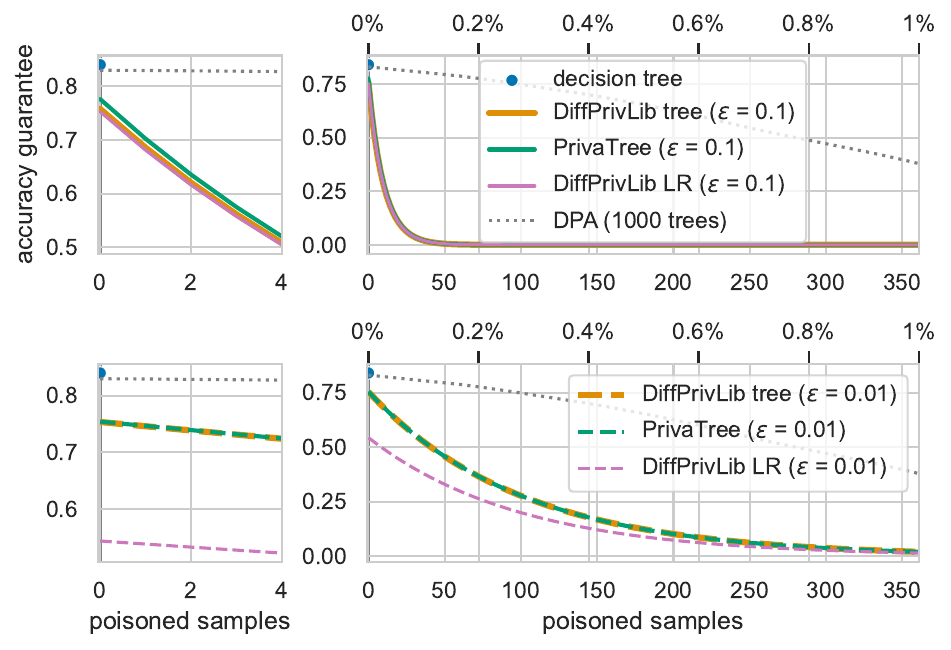}
         \caption{\textit{adult} dataset (45222 samples, 14 features)}
     \end{subfigure}
     \hfill
     \begin{subfigure}[b]{0.48\textwidth}
         \centering
         \includegraphics[width=\textwidth]{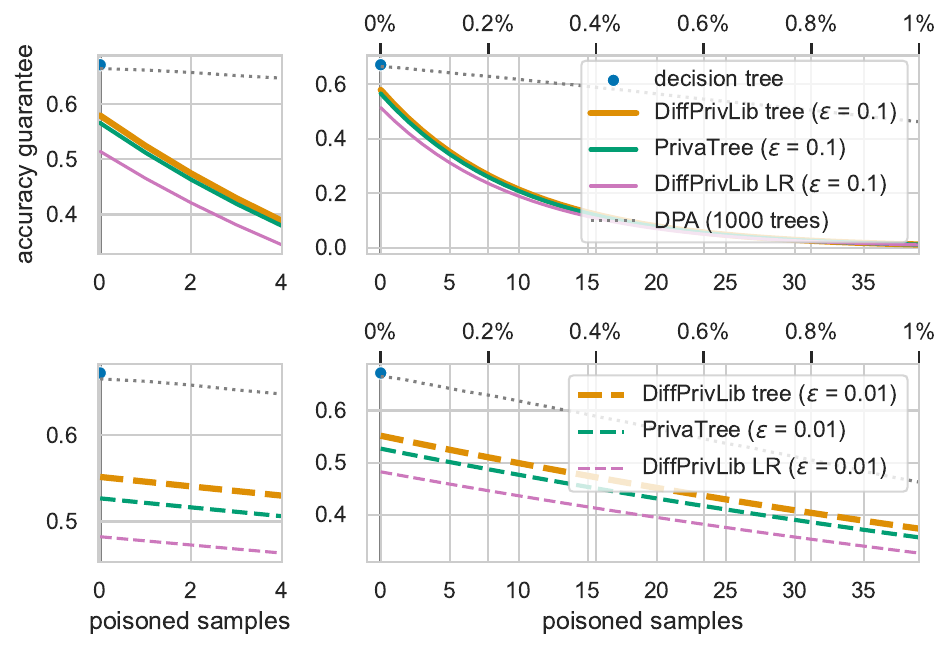}
         \caption{\textit{compas-two-years} (4966 samples, 11 features)}
     \end{subfigure}
        \caption{Guaranteed accuracy when varying the number of poisoned samples to up to 1\% of the dataset size for trees of depth 4. A DPA ensemble of 1000 trees offers strong guarantees but is not interpretable nor does it maintain privacy. For differentially-private learners, there is a trade-off between clean accuracy and the quality of the robustness guarantee.}
        \label{fig:varying-poison}
\end{figure*}

\begin{figure}[tb]
     \centering
     \begin{subfigure}[c]{0.1\columnwidth}
         \centering
         \includegraphics[width=\textwidth]{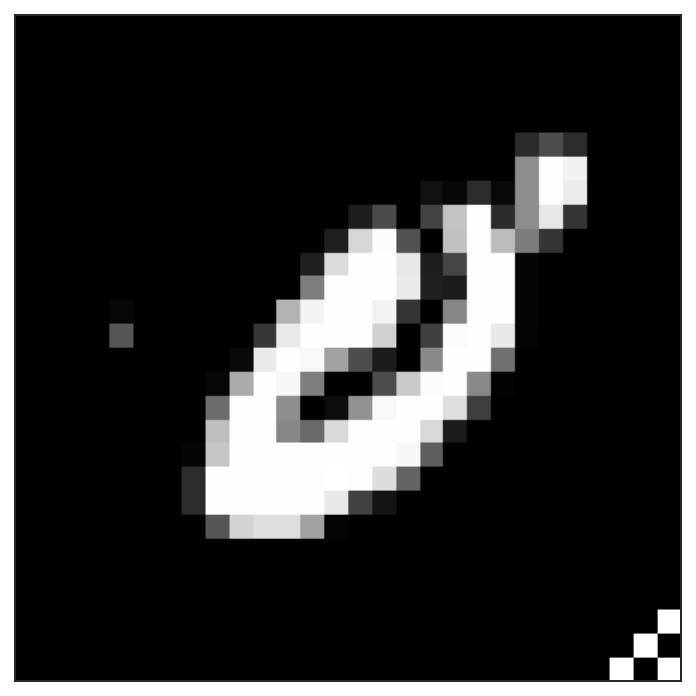}
     \end{subfigure}
     \hfill
     \begin{subfigure}[c]{0.88\columnwidth}
         \centering
         \includegraphics[width=\textwidth]{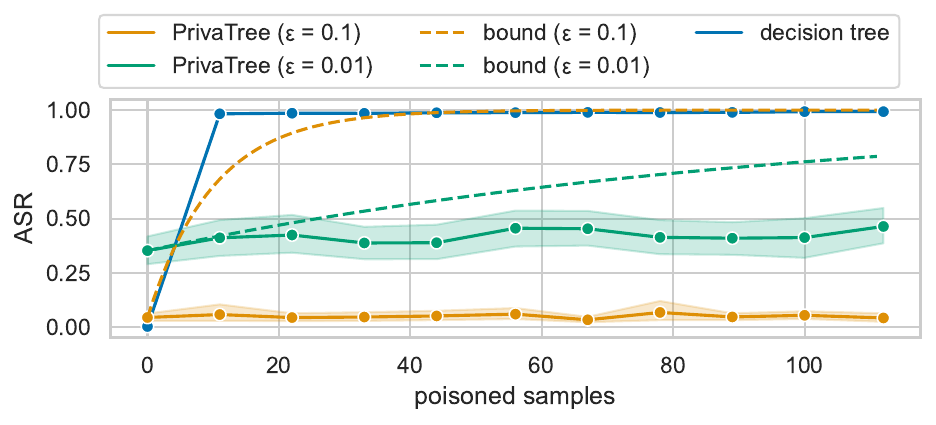}
     \end{subfigure}
        \caption{An adversary injects zeros with a trigger pattern to create a backdoor for class 1. (left) Poisoned sample of a 0 labeled as 1 with a trigger in the bottom-right. (right) The attack success rate when varying the number of poisoned samples in MNIST 0 vs 1 out of 11,200 train samples. $\epsilon {=} 0.01$ offers a tighter bound than $\epsilon {=} 0.1$ but in practice, both values defend well against the backdoor attack on this dataset.}
        \label{fig:backdoor-attack}
\end{figure}

\subsection{Poisoning robustness guarantees for tabular data}
Recall from Section \ref{sec:poisoning-robustness} that differentially-private learners offer guarantees on the loss of accuracy incurred by attackers who poison the training dataset. In Figure \ref{fig:varying-poison}, we visualized the guaranteed accuracy when varying the number of poisoned sampled in the dataset from 0\% to 1\% of the original train set size. These guarantees were determined by estimating the expected test accuracy with 50 random train-test splits and then computing the guarantees with Corollary \ref{cor:accuracy-guarantee}. For comparison, we visualize the guarantee for DPA with an ensemble of 1000 depth 4 trees, but this method does not protect privacy and is not interpretable. The strength of the poisoning robustness guarantee for private learners is determined largely by the choice of $\epsilon$ and to a lesser extent the accuracy on the clean dataset. Therefore, there is an important trade-off between clean data accuracy and robustness guarantee.

Since previous works have considered the poisoning robustness guarantees of private logistic regression as an interpretable model, we compare against this method in more detail. In Table \ref{tab:poison-performance}, we display the accuracy and guarantees at various poison levels for private logistic regression and PrivaTree. We evaluated the models on numerical datasets as this will prevent differences in performance due to the way categorical features are encoded for logistic regression and PrivaTree. Results on data with categorical features can be found in the appendix. It is clear that there is a trade-off between accuracy and poison guarantee for both methods as guarantees at levels of 0.1\% data poisoning are already always better when $\epsilon = 0.01$, whereas test accuracy is always better when $\epsilon = 0.1$. PrivaTrees outperform private logistic regression at equal privacy levels as these complex datasets likely contain non-linear patterns that are better captured by decision trees. Therefore, in the future, one should consider private decision trees as an alternative to private logistic regression when learning interpretable models with poisoning robustness guarantees.

\subsection{Backdoor robustness on MNIST}

To demonstrate the effectiveness of differentially private decision tree learners at mitigating data poisoning attacks, we evaluated backdoor attacks on the MNIST 0 vs 1 dataset. Specifically, we repeat the experiment from Badnets~\cite{gu2017badnets} in which the adversary adds a fixed trigger pattern to the bottom right corner of the image in an attempt to force zeros to be classified as ones. To achieve this, the adversary copies $x$ zeros, adds the trigger pattern to these images, and adds the copies to the training set with label 1. An example of a zero with a trigger pattern is shown in Figure \ref{fig:backdoor-attack}. To measure the robustness of models against the backdoor, we compute the Attack Success Rate (ASR), which is the percentage of test samples with label 0 that are predicted as 1 when the trigger pattern is added.
In Figure \ref{fig:backdoor-attack}, we plot the ASR of a regular decision tree and PrivaTrees with privacy budgets 0.1 and 0.01 and their bounds computed with Corollary \ref{cor:asr-guarantee} against a varying number of poisoned samples ranging between 0\% to 1\% of the dataset. All trees were trained on 50 train-test splits and had a depth of 4. With only 0.01\% of the train set poisoned, regular decision trees already suffer from an ASR of almost 100\% whereas PrivaTrees on average stay at an ASR of under 20\% for the entire range. While the bound for $\epsilon = 0.01$ is much tighter than the bound for $\epsilon = 0.1$, PrivaTrees perform well in practice for both settings.

Strong privacy and robustness guarantees come at the cost of utility, i.e. clean dataset accuracy. The test accuracy scores without poisoned samples for the models in Figure \ref{fig:backdoor-attack} were 99.6\% for the regular decision tree, 95.1\% for PrivaTree with $\epsilon{=}0.1$, and 77.7\% for PrivaTree with $\epsilon{=}0.01$.

%% file: resources/table_private_tree_results.tex
\begingroup
\begin{table*}[tb]
\centering
\caption{5-fold cross-validated mean test accuracy and standard errors at $\epsilon {=} 0.1$ for trees of depth 4. PrivaTree* uses non-private quantiles, DPGDF only ran on categorical features. PrivaTree outperforms existing methods on most datasets.}
\label{tab:performance-comparison}
\begin{tabular}{l|c|cc|ccc}
\toprule
\textbf{OpenML dataset} & \textbf{decision tree} & \textbf{BDPT} & \textbf{PrivaTree*} & \textbf{DPGDF} & \textbf{DiffPrivLib} & \textbf{PrivaTree} \\ 
 & no privacy & \multicolumn{2}{c|}{leaking numerical splits} & \multicolumn{3}{c}{differential privacy} \\ \midrule \multicolumn{7}{c}{Numerical data} \\ \midrule
Bioresponse & .711 \tiny $\pm$ .006 & .508 \tiny $\pm$ .009 & \textbf{.553} \tiny $\pm$ .021 & - & \textbf{.517} \tiny $\pm$ .008 & .505 \tiny $\pm$ .003 \\
Diabetes130US & .606 \tiny $\pm$ .001 & .515 \tiny $\pm$ .006 & \textbf{.556} \tiny $\pm$ .010 & - & .522 \tiny $\pm$ .009 & \textbf{.529} \tiny $\pm$ .011 \\
Higgs & .657 \tiny $\pm$ .001 & timeout & \textbf{.646} \tiny $\pm$ .002 & - & .514 \tiny $\pm$ .004 & \textbf{.593} \tiny $\pm$ .003 \\
MagicTelescope & .781 \tiny $\pm$ .006 & .500 \tiny $\pm$ .000 & \textbf{.687} \tiny $\pm$ .025 & - & \textbf{.663} \tiny $\pm$ .038 & .658 \tiny $\pm$ .025 \\
MiniBooNE & .871 \tiny $\pm$ .001 & .500 \tiny $\pm$ .000 & \textbf{.806} \tiny $\pm$ .008 & - & .503 \tiny $\pm$ .001 & \textbf{.763} \tiny $\pm$ .011 \\
bank-marketing & .771 \tiny $\pm$ .005 & .499 \tiny $\pm$ .000 & \textbf{.689} \tiny $\pm$ .019 & - & \textbf{.560} \tiny $\pm$ .022 & .555 \tiny $\pm$ .031 \\
california & .783 \tiny $\pm$ .002 & .500 \tiny $\pm$ .000 & \textbf{.716} \tiny $\pm$ .014 & - & .593 \tiny $\pm$ .039 & \textbf{.708} \tiny $\pm$ .031 \\
covertype & .740 \tiny $\pm$ .001 & .501 \tiny $\pm$ .001 & \textbf{.740} \tiny $\pm$ .002 & - & .531 \tiny $\pm$ .007 & \textbf{.728} \tiny $\pm$ .001 \\
credit & .748 \tiny $\pm$ .001 & .500 \tiny $\pm$ .000 & \textbf{.689} \tiny $\pm$ .025 & - & \textbf{.528} \tiny $\pm$ .010 & .523 \tiny $\pm$ .014 \\
default-of-credit-card-clients & .704 \tiny $\pm$ .002 & .500 \tiny $\pm$ .000 & \textbf{.593} \tiny $\pm$ .015 & - & .565 \tiny $\pm$ .030 & \textbf{.573} \tiny $\pm$ .028 \\
electricity & .734 \tiny $\pm$ .002 & .500 \tiny $\pm$ .000 & \textbf{.728} \tiny $\pm$ .004 & - & \textbf{.633} \tiny $\pm$ .020 & .617 \tiny $\pm$ .011 \\
eye\_movements & .574 \tiny $\pm$ .003 & .500 \tiny $\pm$ .001 & \textbf{.517} \tiny $\pm$ .006 & - & \textbf{.516} \tiny $\pm$ .006 & .510 \tiny $\pm$ .005 \\
heloc & .704 \tiny $\pm$ .004 & .526 \tiny $\pm$ .009 & \textbf{.605} \tiny $\pm$ .023 & - & .559 \tiny $\pm$ .020 & \textbf{.587} \tiny $\pm$ .027 \\
house\_16H & .815 \tiny $\pm$ .004 & .500 \tiny $\pm$ .000 & \textbf{.690} \tiny $\pm$ .019 & - & .533 \tiny $\pm$ .007 & \textbf{.593} \tiny $\pm$ .032 \\
jannis & .715 \tiny $\pm$ .002 & .500 \tiny $\pm$ .000 & \textbf{.651} \tiny $\pm$ .007 & - & .530 \tiny $\pm$ .005 & \textbf{.633} \tiny $\pm$ .012 \\
pol & .929 \tiny $\pm$ .003 & .538 \tiny $\pm$ .023 & \textbf{.663} \tiny $\pm$ .048 & - & .569 \tiny $\pm$ .032 & \textbf{.578} \tiny $\pm$ .042 \\ \midrule
\multicolumn{7}{c}{Numerical \& categorical data} \\ \midrule
albert & .640 \tiny $\pm$ .002 & .500 \tiny $\pm$ .000 & \textbf{.624} \tiny $\pm$ .003 & .511 \tiny $\pm$ .005 & .532 \tiny $\pm$ .018 & \textbf{.533} \tiny $\pm$ .016 \\
compas-two-years & .672 \tiny $\pm$ .006 & .576 \tiny $\pm$ .021 & \textbf{.617} \tiny $\pm$ .016 & .568 \tiny $\pm$ .006 & \textbf{.585} \tiny $\pm$ .012 & .574 \tiny $\pm$ .012 \\
covertype & .756 \tiny $\pm$ .000 & .606 \tiny $\pm$ .007 & \textbf{.745} \tiny $\pm$ .002 & .535 \tiny $\pm$ .008 & .545 \tiny $\pm$ .022 & \textbf{.744} \tiny $\pm$ .002 \\
default-of-credit-card-clients & .707 \tiny $\pm$ .004 & .500 \tiny $\pm$ .000 & \textbf{.562} \tiny $\pm$ .013 & .528 \tiny $\pm$ .005 & \textbf{.584} \tiny $\pm$ .026 & .572 \tiny $\pm$ .033 \\
electricity & .732 \tiny $\pm$ .002 & .516 \tiny $\pm$ .015 & \textbf{.726} \tiny $\pm$ .004 & .517 \tiny $\pm$ .002 & .539 \tiny $\pm$ .004 & \textbf{.616} \tiny $\pm$ .006 \\
eye\_movements & .579 \tiny $\pm$ .001 & .506 \tiny $\pm$ .006 & \textbf{.533} \tiny $\pm$ .009 & \textbf{.531} \tiny $\pm$ .007 & .506 \tiny $\pm$ .004 & .511 \tiny $\pm$ .012 \\
road-safety & .728 \tiny $\pm$ .001 & .685 \tiny $\pm$ .001 & \textbf{.704} \tiny $\pm$ .002 & .685 \tiny $\pm$ .001 & .554 \tiny $\pm$ .032 & \textbf{.692} \tiny $\pm$ .004 \\ \midrule
\multicolumn{7}{c}{UCI datasets (numerical \& categorical)} \\ \midrule
adult & .840 \tiny $\pm$ .001 & .752 \tiny $\pm$ .000 & \textbf{.776} \tiny $\pm$ .014 & .751 \tiny $\pm$ .000 & .758 \tiny $\pm$ .005 & \textbf{.782} \tiny $\pm$ .012 \\
breast-w & .950 \tiny $\pm$ .007 & .669 \tiny $\pm$ .028 & \textbf{.902} \tiny $\pm$ .013 & - & \textbf{.921} \tiny $\pm$ .011 & .887 \tiny $\pm$ .038 \\
diabetes & .734 \tiny $\pm$ .006 & \textbf{.622} \tiny $\pm$ .023 & .611 \tiny $\pm$ .042 & - & \textbf{.678} \tiny $\pm$ .018 & .654 \tiny $\pm$ .008 \\
mushroom & .971 \tiny $\pm$ .001 & \textbf{.855} \tiny $\pm$ .043 & .788 \tiny $\pm$ .035 & .694 \tiny $\pm$ .045 & .759 \tiny $\pm$ .022 & \textbf{.867} \tiny $\pm$ .036 \\
nursery & 1.000 \tiny $\pm$ .000 & .998 \tiny $\pm$ .002 & \textbf{1.000} \tiny $\pm$ .000 & .701 \tiny $\pm$ .017 & .753 \tiny $\pm$ .063 & \textbf{1.000} \tiny $\pm$ .000 \\
vote & .944 \tiny $\pm$ .013 & .647 \tiny $\pm$ .057 & \textbf{.857} \tiny $\pm$ .042 & .596 \tiny $\pm$ .098 & \textbf{.854} \tiny $\pm$ .050 & .749 \tiny $\pm$ .054 \\
\midrule
\emph{Total average} & .762 & .572 & \textbf{.689} & - & .600 & \textbf{.649} \\
\bottomrule
\end{tabular}
\end{table*}
\endgroup

%% file: resources/table_robustness_guarantees.tex
\begingroup
\begin{table*}[tb]
\centering
\caption{5-fold cross-validated mean test accuracy and poisoning accuracy guarantee against a percentage of poisoned samples on numerical datasets. Stronger privacy provides stronger poisoning robustness but comes at the cost of clean dataset accuracy. PrivaTree outperforms private logistic regression on this benchmark.}
\label{tab:poison-performance}

\begin{minipage}{.497\linewidth}
\centering
\begin{tabular}{@{}lllrrrr@{}}
\toprule
 &  &  &  & \multicolumn{3}{c}{\textbf{poison guarantee}} \\
\textbf{dataset} & \textbf{method} & \textbf{$\epsilon$} & \textbf{acc.} & \textbf{0.1\%} & \textbf{0.5\%} & \textbf{1\%} \\ \midrule
Bioresponse & PrivaTree & .01 & .495 & .486 & .435 & .378 \\
 &  & .1 & \textbf{.515} & .421 & .140 & .035 \\
 & diffprivlib LR & .01 & .509 & \textbf{.499} & \textbf{.447} & \textbf{.389} \\
 &  & .1 & .501 & .411 & .137 & .034 \\ \midrule
Diabetes130US & PrivaTree & .01 & .509 & \textbf{.291} & \textbf{.030} & \textbf{.002} \\
 &  & .1 & \textbf{.543} & .002 & .000 & .000 \\
 & diffprivlib LR & .01 & .484 & .276 & .028 & .002 \\
 &  & .1 & .506 & .002 & .000 & .000 \\ \midrule
Higgs & PrivaTree & .01 & .537 & .000 & .000 & .000 \\
 &  & .1 & \textbf{.596} & .000 & .000 & .000 \\
 & diffprivlib LR & .01 & .491 & .000 & .000 & .000 \\
 &  & .1 & \textbf{.596} & .000 & .000 & .000 \\ \midrule
MagicTelescope & PrivaTree & .01 & .585 & \textbf{.529} & \textbf{.344} & \textbf{.201} \\
 &  & .1 & \textbf{.654} & .241 & .003 & .000 \\
 & diffprivlib LR & .01 & .566 & .512 & .333 & .194 \\
 &  & .1 & .601 & .221 & .003 & .000 \\ \midrule
MiniBooNE & PrivaTree & .01 & .502 & .281 & .027 & .001 \\
 &  & .1 & \textbf{.763} & .002 & .000 & .000 \\
 & diffprivlib LR & .01 & .587 & \textbf{.328} & \textbf{.032} & \textbf{.002} \\
 &  & .1 & .608 & .002 & .000 & .000 \\ \midrule
bank-marketing & PrivaTree & .01 & .522 & .482 & .343 & .225 \\
 &  & .1 & \textbf{.650} & .292 & .010 & .000 \\
 & diffprivlib LR & .01 & .537 & \textbf{.496} & \textbf{.353} & \textbf{.232} \\
 &  & .1 & .484 & .217 & .007 & .000 \\ \midrule
california & PrivaTree & .01 & .533 & \textbf{.454} & \textbf{.235} & \textbf{.102} \\
 &  & .1 & \textbf{.692} & .140 & .000 & .000 \\
 & diffprivlib LR & .01 & .517 & .440 & .228 & .099 \\
 &  & .1 & .434 & .088 & .000 & .000 \\ \midrule
covertype & PrivaTree & .01 & .727 & \textbf{.008} & .000 & .000 \\
 &  & .1 & \textbf{.731} & .000 & .000 & .000 \\
 & diffprivlib LR & .01 & .584 & .006 & .000 & .000 \\
 &  & .1 & .615 & .000 & .000 & .000 \\ \bottomrule
\end{tabular}
\end{minipage}
\begin{minipage}{.497\linewidth}
\centering
\begin{tabular}{@{}lllrrrr@{}}
\toprule
 &  &  &  & \multicolumn{3}{c}{\textbf{poison guarantee}} \\
\textbf{dataset} & \textbf{method} & \textbf{$\epsilon$} & \textbf{acc.} & \textbf{0.1\%} & \textbf{0.5\%} & \textbf{1\%} \\ \midrule
credit & PrivaTree & .01 & .508 & \textbf{.446} & \textbf{.262} & \textbf{.134} \\
 &  & .1 & \textbf{.529} & .144 & .001 & .0 \\
 & diffprivlib LR & .01 & .472 & .415 & .244 & .125 \\
 &  & .1 & .471 & .128 & .001 & .0 \\ \midrule
default-of-credit. & PrivaTree & .01 & .523 & \textbf{.473} & \textbf{.308} & \textbf{.181} \\
 &  & .1 & \textbf{.579} & .213 & .003 & .0 \\
 & diffprivlib LR & .01 & .507 & .459 & .299 & .176 \\
 &  & .1 & .547 & .201 & .003 & .0 \\ \midrule
electricity & PrivaTree & .01 & .563 & \textbf{.417} & \textbf{.122} & \textbf{.026} \\
 &  & .1 & \textbf{.611} & .03 & .0 & .0 \\
 & diffprivlib LR & .01 & .419 & .31 & .091 & .019 \\
 &  & .1 & .528 & .026 & .0 & .0 \\ \midrule
eye\_movements & PrivaTree & .01 & \textbf{.507} & \textbf{.478} & \textbf{.376} & \textbf{.279} \\
 &  & .1 & .505 & .277 & .025 & .001 \\
 & diffprivlib LR & .01 & .487 & .458 & .361 & .267 \\
 &  & .1 & .506 & .278 & .025 & .001 \\ \midrule
heloc & PrivaTree & .01 & .517 & .478 & .347 & .232 \\
 &  & .1 & \textbf{.570} & .256 & .01 & .0 \\
 & diffprivlib LR & .01 & .54 & \textbf{.498} & \textbf{.362} & \textbf{.243} \\
 &  & .1 & .538 & .242 & .01 & .0 \\ \midrule
house\_16H & PrivaTree & .01 & .567 & \textbf{.513} & \textbf{.334} & \textbf{.194} \\
 &  & .1 & \textbf{.581} & .214 & .003 & .0 \\
 & diffprivlib LR & .01 & .521 & .472 & .307 & .179 \\
 &  & .1 & .553 & .204 & .003 & .0 \\ \midrule
jannis & PrivaTree & .01 & .538 & \textbf{.34} & \textbf{.054} & \textbf{.005} \\
 &  & .1 & \textbf{.653} & .007 & .0 & .0 \\
 & diffprivlib LR & .01 & .51 & .322 & .051 & .005 \\
 &  & .1 & .552 & .006 & .0 & .0 \\ \midrule
pol & PrivaTree & .01 & .545 & \textbf{.503} & \textbf{.365} & \textbf{.245} \\
 &  & .1 & \textbf{.549} & .247 & .01 & .0 \\
 & diffprivlib LR & .01 & .535 & .494 & .359 & .241 \\
 &  & .1 & .526 & .236 & .01 & .0 \\ \bottomrule
\end{tabular}
\end{minipage}

\end{table*}
\endgroup

%% file: sections/06_discussion.tex
\section{Discussion}

Some limitations can arise when applying PrivaTree to real-world scenarios. In our experiments, we compared the performance of various differentially-private decision tree learners on UCI data and the tabular data benchmark. While some UCI datasets are too easy, the tabular benchmark~\cite{grinsztajn2022tree} was specifically curated so that decision trees alone do not easily score perfectly. For real use cases, data could be easier to classify. Also, while our code supports multiclass classification these benchmarks contain only binary classification tasks and so we have not evaluated models in this setting. It is worth noting that in the multiclass case methods such as logistic regression need to train a separate model for each class which hinders interpretability while decision trees still learn one tree.

As is typical, we assume that the range of numerical features, the set of possible categorical values, and the set of class labels are public knowledge. Additionally, we use the number of samples in the training set to select an efficient value for $\epsilon_\text{leaf}$ which assumes that we can publish information on the dataset size while some other works protect this value. One mitigation is to use rough estimates of the dataset size, e.g. only the order of magnitude.

Privacy and robustness in machine learning are important topics as models trained on user data are continuously deployed in the world. Differential privacy is a promising technique for this and we improve the performance of decision trees at high differential privacy levels. However, we want to warn against over-optimism as differential privacy is not a silver bullet for AI security. Engineers must take into account in which context models are deployed to decide what constitutes an acceptable privacy risk and must take into account what attributes are not protected by our method. Regarding poisoning robustness, it is also vital to understand the threat model that is being defended against to verify that the robustness guarantees for differentially-private learners apply. We hope that improvements in the privacy-utility trade-off for interpretable learners, such as the ones we propose, will increase the adoption of interpretable and private methods to improve the trustworthiness of machine learning systems.

\section{Conclusion}
In this paper, we proposed a new algorithm for training differentially private decision trees called PrivaTree. PrivaTree uses histograms with permute-and-flip for node selection, the permute-and-flip mechanism for leaf labeling, and a more efficient privacy budget distribution method to improve the privacy-accuracy trade-off. Our experiments on two benchmarks demonstrate that PrivaTree on average scores better than existing works at a fixed privacy budget. Moreover, we investigated the poisoning robustness guarantees for differentially-private learners and also applied this to the setting of backdoor attacks. On the MNIST 0 vs 1 task, differentially-private decision trees provide a fivefold reduction in attack success rate compared to decision trees trained without privacy. 
While our work makes progress in privacy budget allocation to improve the privacy-utility trade-off, follow-up work may further improve this trade-off in the very high privacy regime.

%% file: sections/appendix.tex
\section{Appendix}

\subsection{Expected worst-case leaf labeling error}

In the main text, we used the fact that we could bound the amount of privacy budget $\epsilon'_\mathrm{leaf}$ needed to label leaves with expected worst-case incurred error at most $\mathcal{E}_{max}$. By limiting the amount of privacy budget for leaf labeling this way we make sure to leave more privacy budget for node selection when possible. We give a short proof of the theorem below.

\begingroup
\def\thetheorem{\ref{cor:max-leaf-error}}
\begin{corollary}
For $K$ classes, $n$ samples and depth $d$ trees, the amount of privacy budget $\epsilon'_\text{leaf}$ needed for labeling leaves with the permute-and-flip mechanism with expected error $\mathbb{E}[\mathcal{E}(\mathcal{M}_{PF}, \vec{N})]$ of at most $\mathcal{E}_{max}$ is:
\begin{equation*}
    \epsilon'_\mathrm{leaf} \leq \frac{2^d \max_p 2 \log(\frac{1}{p}) \left( 1 - \frac{1 - (1 - p)^K}{K p} \right)}{n \; \mathcal{E}_{max}}\;,
\end{equation*}
\end{corollary}
\addtocounter{theorem}{-1}
\endgroup
\begin{proof}
Recall Proposition 4 from the permute-and-flip paper~\cite{mckenna2020permute} that for a vector of candidates with errors $\vec{q} \in \mathbb{R}^K$ the expected worst case error $\mathbb{E}[\mathcal{E}(\mathcal{M}_{PF}, \vec{q})]$ occurs when all but one candidates share the same error $c$ (and thus share probability of being selected $p = \exp(\frac{\epsilon}{2 \Delta} c)$). The expected errors for such vectors of this form are:
\begin{equation*}
    \mathbb{E}[\mathcal{E}(\mathcal{M}_{PF}, \vec{q})] = \frac{2 \Delta}{\epsilon} \log \left(\frac{1}{p} \right) \left( 1 - \frac{1 - (1 - p)^K}{K p} \right) \;.
\end{equation*}
The worst-case expected error can be found by maximizing over $p \in [0, 1]$, i.e. after substituting sensitivity $\Delta {=} 1$ and $\epsilon {=} \epsilon'_\text{leaf}$:
\begin{equation*}
    \max_p \frac{2}{\epsilon'_\text{leaf}} \log \left(\frac{1}{p} \right) \left( 1 - \frac{1 - (1 - p)^K}{K p} \right) \;.
\end{equation*}

Now we do not want to bound the total error but the percentage error so we divide by $n$ samples, and since we can incur error for every leaf we multiply by $2^d$. After bounding by the user-specified value $\mathcal{E}_{max}$ we find a sufficient value for $\epsilon'_\text{leaf}$:
\begin{align*}
    \frac{2^d \max_p \frac{2}{\epsilon'_\text{leaf}} \log(\frac{1}{p}) \left( 1 - \frac{1 - (1 - p)^K}{K p} \right)}{n} &\leq \mathcal{E}_{max} \;, \\
    \frac{2^d \max_p 2 \log(\frac{1}{p}) \left( 1 - \frac{1 - (1 - p)^K}{K p} \right)}{n \; \mathcal{E}_{max}} &\leq \epsilon'_\text{leaf} \;.
\qedhere
\end{align*}
\end{proof}

In our implementation, we solve the maximization term numerically using Scipy.

\subsection{Dataset properties}
We summarize the properties of the datasets that we included in our benchmark in Table~\ref{tab:datasets}, dataset sizes are displayed after removing rows with missing values. Since UCI datasets are imbalanced, private models often perform worse than guessing the majority class for low privacy budgets.

\begingroup
\renewcommand{\arraystretch}{1.11}

\begin{table*}[tb]
\centering
\caption{Properties of the datasets used in this work. Rows with missing values were removed. UCI datasets are often imbalanced.}
\label{tab:datasets}
\begin{tabular}{l|rrrr}
\toprule
\textbf{Dataset} & \textbf{Samples} & \textbf{Features} & \textbf{Categorical} & \textbf{Majority} \\
&  &  & \textbf{features} & \textbf{class share} \\ \midrule  
\multicolumn{5}{c}{Numerical data} \\ \midrule
Bioresponse & 3,434 & 419 & 0 & 0.500 \\
Diabetes130US & 71,090 & 7 & 0 & 0.500 \\
Higgs & 940,160 & 24 & 0 & 0.500 \\
MagicTelescope & 13,376 & 10 & 0 & 0.500 \\
MiniBooNE & 72,998 & 50 & 0 & 0.500 \\
bank-marketing & 10,578 & 7 & 0 & 0.500 \\
california & 20,634 & 8 & 0 & 0.500 \\
covertype & 566,602 & 10 & 0 & 0.500 \\
credit & 16,714 & 10 & 0 & 0.500 \\
default-of-credit-card-clients & 13,272 & 20 & 0 & 0.500 \\
electricity & 38,474 & 7 & 0 & 0.500 \\
eye\_movements & 7,608 & 20 & 0 & 0.500 \\
heloc & 10,000 & 22 & 0 & 0.500 \\
house\_16H & 13,488 & 16 & 0 & 0.500 \\
jannis & 57,580 & 54 & 0 & 0.500 \\
pol & 10,082 & 26 & 0 & 0.500 \\ \midrule  
\multicolumn{5}{c}{Numerical \& categorical data} \\ \midrule
albert & 58,252 & 31 & 10 & 0.500 \\
compas-two-years & 4,966 & 11 & 8 & 0.500 \\
covertype & 423,680 & 54 & 44 & 0.500 \\
default-of-credit-card-clients & 13,272 & 21 & 1 & 0.500 \\
electricity & 38,474 & 8 & 1 & 0.500 \\
eye\_movements & 7,608 & 23 & 3 & 0.500 \\
road-safety & 111,762 & 32 & 3 & 0.500 \\ \midrule  
\multicolumn{5}{c}{UCI datasets (numerical \& categorical)} \\ \midrule
adult & 45,222 & 14 & 8 & 0.752 \\
breast-w & 683 & 9 & 0 & 0.650 \\
diabetes & 768 & 8 & 0 & 0.651 \\
mushroom & 5,644 & 22 & 22 & 0.618 \\
nursery & 12,960 & 8 & 8 & 0.667 \\
vote & 232 & 16 & 16 & 0.534 \\ \bottomrule
\end{tabular}
\end{table*}
\endgroup

\subsection{Runtime comparison}

We measured the runtime of all methods when performing 5-fold cross validations and display the results in Table~\ref{tab:runtime}. Regular decision trees run in milliseconds benefitting from the fast implementation by Scikit-learn~\cite{pedregosa2011scikit}. DiffPrivLib does not need to perform node selection operations and thus only spends milliseconds on propagating data points to the leaves and labeling them. DPGDF, BDPT and PrivaTree usually run in seconds, however, on large numerical datasets BDPT takes minutes.

\begingroup
\renewcommand{\arraystretch}{1.11}

\begin{table*}[tb]
\setlength{\tabcolsep}{3.5pt}
\centering
\caption{Mean runtimes in seconds and standard errors at $\epsilon {=} 0.1$ for trees of depth 4 with 5 repetitions.}
\label{tab:runtime}
\begin{tabular}{l|r|rr|rrr}
\toprule
\textbf{OpenML dataset} & \textbf{decision tree} & \textbf{BDPT} & \textbf{PrivaTree*} & \textbf{DPGDF} & \textbf{DiffPrivLib} & \textbf{PrivaTree} \\ 
 & no privacy & \multicolumn{2}{c|}{leaking numerical splits} & \multicolumn{3}{c}{differential privacy} \\ \midrule \multicolumn{7}{c}{Numerical data} \\ \midrule
Bioresponse & <1 \tiny $\pm$ 0 & 2 \tiny $\pm$ 0 & 1 \tiny $\pm$ 0 & - & <1 \tiny $\pm$ 0 & 1 \tiny $\pm$ 0 \\
Diabetes130US & <1 \tiny $\pm$ 0 & <1 \tiny $\pm$ 0 & <1 \tiny $\pm$ 0 & - & <1 \tiny $\pm$ 0 & <1 \tiny $\pm$ 0 \\
Higgs & 9 \tiny $\pm$ 0 & - & 2 \tiny $\pm$ 0 & - & 6 \tiny $\pm$ 0 & 1 \tiny $\pm$ 0 \\
MagicTelescope & <1 \tiny $\pm$ 0 & 2 \tiny $\pm$ 0 & <1 \tiny $\pm$ 0 & - & <1 \tiny $\pm$ 0 & <1 \tiny $\pm$ 0 \\
MiniBooNE & 2 \tiny $\pm$ 0 & 366 \tiny $\pm$ 15 & <1 \tiny $\pm$ 0 & - & 1 \tiny $\pm$ 0 & <1 \tiny $\pm$ 0 \\
bank-marketing & <1 \tiny $\pm$ 0 & <1 \tiny $\pm$ 0 & <1 \tiny $\pm$ 0 & - & <1 \tiny $\pm$ 0 & <1 \tiny $\pm$ 0 \\
california & <1 \tiny $\pm$ 0 & 1 \tiny $\pm$ 0 & <1 \tiny $\pm$ 0 & - & <1 \tiny $\pm$ 0 & <1 \tiny $\pm$ 0 \\
covertype & 1 \tiny $\pm$ 0 & 9 \tiny $\pm$ 0 & 1 \tiny $\pm$ 0 & - & 4 \tiny $\pm$ 0 & <1 \tiny $\pm$ 0 \\
credit & <1 \tiny $\pm$ 0 & 1 \tiny $\pm$ 0 & <1 \tiny $\pm$ 0 & - & <1 \tiny $\pm$ 0 & <1 \tiny $\pm$ 0 \\
default-of-credit-card-clients & <1 \tiny $\pm$ 0 & 1 \tiny $\pm$ 0 & <1 \tiny $\pm$ 0 & - & <1 \tiny $\pm$ 0 & <1 \tiny $\pm$ 0 \\
electricity & <1 \tiny $\pm$ 0 & <1 \tiny $\pm$ 0 & <1 \tiny $\pm$ 0 & - & <1 \tiny $\pm$ 0 & <1 \tiny $\pm$ 0 \\
eye\_movements & <1 \tiny $\pm$ 0 & 1 \tiny $\pm$ 0 & <1 \tiny $\pm$ 0 & - & <1 \tiny $\pm$ 0 & <1 \tiny $\pm$ 0 \\
heloc & <1 \tiny $\pm$ 0 & <1 \tiny $\pm$ 0 & <1 \tiny $\pm$ 0 & - & <1 \tiny $\pm$ 0 & <1 \tiny $\pm$ 0 \\
house\_16H & <1 \tiny $\pm$ 0 & 2 \tiny $\pm$ 0 & <1 \tiny $\pm$ 0 & - & <1 \tiny $\pm$ 0 & <1 \tiny $\pm$ 0 \\
jannis & 1 \tiny $\pm$ 0 & 173 \tiny $\pm$ 14 & <1 \tiny $\pm$ 0 & - & 1 \tiny $\pm$ 0 & <1 \tiny $\pm$ 0 \\
pol & <1 \tiny $\pm$ 0 & <1 \tiny $\pm$ 0 & <1 \tiny $\pm$ 0 & - & <1 \tiny $\pm$ 0 & <1 \tiny $\pm$ 0 \\ \midrule
\multicolumn{7}{c}{Numerical \& categorical data} \\ \midrule
albert & <1 \tiny $\pm$ 0 & 2 \tiny $\pm$ 0 & <1 \tiny $\pm$ 0 & <1 \tiny $\pm$ 0 & <1 \tiny $\pm$ 0 & <1 \tiny $\pm$ 0 \\
compas-two-years & <1 \tiny $\pm$ 0 & <1 \tiny $\pm$ 0 & <1 \tiny $\pm$ 0 & <1 \tiny $\pm$ 0 & <1 \tiny $\pm$ 0 & <1 \tiny $\pm$ 0 \\
covertype & 2 \tiny $\pm$ 0 & 17 \tiny $\pm$ 1 & 1 \tiny $\pm$ 0 & 1 \tiny $\pm$ 0 & 4 \tiny $\pm$ 0 & 1 \tiny $\pm$ 0 \\
default-of-credit-card-clients & <1 \tiny $\pm$ 0 & 1 \tiny $\pm$ 0 & <1 \tiny $\pm$ 0 & <1 \tiny $\pm$ 0 & <1 \tiny $\pm$ 0 & <1 \tiny $\pm$ 0 \\
electricity & <1 \tiny $\pm$ 0 & <1 \tiny $\pm$ 0 & <1 \tiny $\pm$ 0 & <1 \tiny $\pm$ 0 & <1 \tiny $\pm$ 0 & <1 \tiny $\pm$ 0 \\
eye\_movements & <1 \tiny $\pm$ 0 & 1 \tiny $\pm$ 0 & <1 \tiny $\pm$ 0 & <1 \tiny $\pm$ 0 & <1 \tiny $\pm$ 0 & <1 \tiny $\pm$ 0 \\
road-safety & 1 \tiny $\pm$ 0 & 54 \tiny $\pm$ 0 & <1 \tiny $\pm$ 0 & <1 \tiny $\pm$ 0 & 1 \tiny $\pm$ 0 & <1 \tiny $\pm$ 0 \\ \midrule
\multicolumn{7}{c}{UCI datasets (numerical \& categorical)} \\ \midrule
adult & <1 \tiny $\pm$ 0 & 1 \tiny $\pm$ 0 & <1 \tiny $\pm$ 0 & <1 \tiny $\pm$ 0 & <1 \tiny $\pm$ 0 & <1 \tiny $\pm$ 0 \\
breast-w & <1 \tiny $\pm$ 0 & <1 \tiny $\pm$ 0 & <1 \tiny $\pm$ 0 & - & <1 \tiny $\pm$ 0 & <1 \tiny $\pm$ 0 \\
diabetes & <1 \tiny $\pm$ 0 & <1 \tiny $\pm$ 0 & <1 \tiny $\pm$ 0 & - & <1 \tiny $\pm$ 0 & <1 \tiny $\pm$ 0 \\
mushroom & <1 \tiny $\pm$ 0 & <1 \tiny $\pm$ 0 & <1 \tiny $\pm$ 0 & <1 \tiny $\pm$ 0 & <1 \tiny $\pm$ 0 & <1 \tiny $\pm$ 0 \\
nursery & <1 \tiny $\pm$ 0 & <1 \tiny $\pm$ 0 & <1 \tiny $\pm$ 0 & <1 \tiny $\pm$ 0 & <1 \tiny $\pm$ 0 & <1 \tiny $\pm$ 0 \\
vote & <1 \tiny $\pm$ 0 & <1 \tiny $\pm$ 0 & <1 \tiny $\pm$ 0 & <1 \tiny $\pm$ 0 & <1 \tiny $\pm$ 0 & <1 \tiny $\pm$ 0 \\ \bottomrule
\end{tabular}
\end{table*}
\endgroup

\subsection{Test accuracy with different privacy budgets}

In the main text we displayed results for depth 4 trees with a privacy budget of $\epsilon = 0.1$. Although this is generally considered as a good value for privacy, we also display results for $\epsilon = 0.01$ and $\epsilon = 1$ in Tables \ref{tab:performance-comparison-001} and \ref{tab:performance-comparison-1} respectively.

\begingroup
\renewcommand{\arraystretch}{1.11}

\begin{table*}[tb]
\setlength{\tabcolsep}{3.5pt}
\centering
\caption{5-fold cross-validated mean test accuracy scores and standard errors at $\epsilon {=} 0.01$ for trees of depth 4. PrivaTree* uses non-private quantiles, DPGDF only ran on categorical features.}
\label{tab:performance-comparison-001}
\begin{tabular}{l|c|cc|ccc}
\toprule
\textbf{OpenML dataset} & \textbf{decision tree} & \textbf{BDPT} & \textbf{PrivaTree*} & \textbf{DPGDF} & \textbf{DiffPrivLib} & \textbf{PrivaTree} \\ 
    & no privacy & \multicolumn{2}{c|}{leaking numerical splits} & \multicolumn{3}{c}{differential privacy} \\ \midrule \multicolumn{7}{c}{Numerical data} \\ \midrule
Bioresponse & .711 \tiny $\pm$ .006 & .500 \tiny $\pm$ .001 & \textbf{.524} \tiny $\pm$ .012 & - & \textbf{.509} \tiny $\pm$ .009 & .501 \tiny $\pm$ .005 \\
Diabetes130US & .606 \tiny $\pm$ .001 & .509 \tiny $\pm$ .008 & \textbf{.554} \tiny $\pm$ .008 & - & .521 \tiny $\pm$ .017 & \textbf{.527} \tiny $\pm$ .008 \\
Higgs & .657 \tiny $\pm$ .001 & timeout & \textbf{.583} \tiny $\pm$ .021 & - & .509 \tiny $\pm$ .004 & \textbf{.565} \tiny $\pm$ .016 \\
MagicTelescope & .781 \tiny $\pm$ .006 & .500 \tiny $\pm$ .000 & \textbf{.624} \tiny $\pm$ .038 & - & .562 \tiny $\pm$ .033 & \textbf{.594} \tiny $\pm$ .033 \\
MiniBooNE & .871 \tiny $\pm$ .001 & .500 \tiny $\pm$ .000 & \textbf{.721} \tiny $\pm$ .009 & - & \textbf{.512} \tiny $\pm$ .010 & .502 \tiny $\pm$ .001 \\
bank-marketing & .771 \tiny $\pm$ .005 & .500 \tiny $\pm$ .001 & \textbf{.556} \tiny $\pm$ .010 & - & \textbf{.556} \tiny $\pm$ .032 & .539 \tiny $\pm$ .040 \\
california & .783 \tiny $\pm$ .002 & .500 \tiny $\pm$ .000 & \textbf{.566} \tiny $\pm$ .023 & - & \textbf{.546} \tiny $\pm$ .010 & .530 \tiny $\pm$ .017 \\
covertype & .740 \tiny $\pm$ .001 & .501 \tiny $\pm$ .001 & \textbf{.698} \tiny $\pm$ .016 & - & .535 \tiny $\pm$ .007 & \textbf{.724} \tiny $\pm$ .003 \\
credit & .748 \tiny $\pm$ .001 & .498 \tiny $\pm$ .002 & \textbf{.630} \tiny $\pm$ .027 & - & \textbf{.507} \tiny $\pm$ .009 & .504 \tiny $\pm$ .002 \\
default-of-credit. & .704 \tiny $\pm$ .002 & .500 \tiny $\pm$ .000 & \textbf{.571} \tiny $\pm$ .012 & - & \textbf{.544} \tiny $\pm$ .034 & .531 \tiny $\pm$ .028 \\
electricity & .734 \tiny $\pm$ .002 & .500 \tiny $\pm$ .000 & \textbf{.608} \tiny $\pm$ .008 & - & .549 \tiny $\pm$ .020 & \textbf{.561} \tiny $\pm$ .023 \\
eye\_movements & .574 \tiny $\pm$ .003 & .500 \tiny $\pm$ .000 & \textbf{.503} \tiny $\pm$ .005 & - & .512 \tiny $\pm$ .005 & \textbf{.520} \tiny $\pm$ .013 \\
heloc & .704 \tiny $\pm$ .004 & .487 \tiny $\pm$ .008 & \textbf{.602} \tiny $\pm$ .012 & - & .526 \tiny $\pm$ .017 & \textbf{.550} \tiny $\pm$ .028 \\
house\_16H & .815 \tiny $\pm$ .004 & .500 \tiny $\pm$ .000 & \textbf{.657} \tiny $\pm$ .038 & - & .562 \tiny $\pm$ .021 & \textbf{.578} \tiny $\pm$ .018 \\
jannis & .715 \tiny $\pm$ .002 & .500 \tiny $\pm$ .000 & \textbf{.570} \tiny $\pm$ .022 & - & \textbf{.578} \tiny $\pm$ .020 & .574 \tiny $\pm$ .017 \\
pol & .929 \tiny $\pm$ .003 & .501 \tiny $\pm$ .001 & \textbf{.575} \tiny $\pm$ .026 & - & \textbf{.547} \tiny $\pm$ .025 & .537 \tiny $\pm$ .026 \\ \midrule
\multicolumn{7}{c}{Numerical \& categorical data} \\ \midrule
albert & .640 \tiny $\pm$ .002 & .500 \tiny $\pm$ .000 & \textbf{.541} \tiny $\pm$ .012 & .505 \tiny $\pm$ .002 & \textbf{.517} \tiny $\pm$ .007 & .509 \tiny $\pm$ .004 \\
compas-two-years & .672 \tiny $\pm$ .006 & .538 \tiny $\pm$ .008 & \textbf{.588} \tiny $\pm$ .017 & .505 \tiny $\pm$ .013 & \textbf{.557} \tiny $\pm$ .014 & .529 \tiny $\pm$ .011 \\
covertype & .756 \tiny $\pm$ .000 & .501 \tiny $\pm$ .001 & \textbf{.739} \tiny $\pm$ .003 & .515 \tiny $\pm$ .009 & .540 \tiny $\pm$ .012 & \textbf{.735} \tiny $\pm$ .001 \\
default-of-credit. & .707 \tiny $\pm$ .004 & .500 \tiny $\pm$ .000 & \textbf{.543} \tiny $\pm$ .014 & .498 \tiny $\pm$ .014 & \textbf{.535} \tiny $\pm$ .020 & .525 \tiny $\pm$ .019 \\
electricity & .732 \tiny $\pm$ .002 & .500 \tiny $\pm$ .000 & \textbf{.615} \tiny $\pm$ .028 & .510 \tiny $\pm$ .001 & .562 \tiny $\pm$ .017 & \textbf{.588} \tiny $\pm$ .020 \\
eye\_movements & .579 \tiny $\pm$ .001 & .500 \tiny $\pm$ .000 & \textbf{.511} \tiny $\pm$ .007 & .496 \tiny $\pm$ .007 & \textbf{.510} \tiny $\pm$ .007 & .498 \tiny $\pm$ .002 \\
road-safety & .728 \tiny $\pm$ .001 & .500 \tiny $\pm$ .000 & \textbf{.522} \tiny $\pm$ .005 & \textbf{.685} \tiny $\pm$ .001 & .554 \tiny $\pm$ .027 & .512 \tiny $\pm$ .005 \\ \midrule
\multicolumn{7}{c}{UCI datasets (numerical \& categorical)} \\ \midrule
adult & .840 \tiny $\pm$ .001 & \textbf{.752} \tiny $\pm$ .000 & .750 \tiny $\pm$ .003 & .744 \tiny $\pm$ .007 & \textbf{.752} \tiny $\pm$ .000 & .751 \tiny $\pm$ .001 \\
breast-w & .950 \tiny $\pm$ .007 & .517 \tiny $\pm$ .066 & \textbf{.757} \tiny $\pm$ .048 & - & \textbf{.846} \tiny $\pm$ .032 & .673 \tiny $\pm$ .086 \\
diabetes & .734 \tiny $\pm$ .006 & \textbf{.612} \tiny $\pm$ .036 & .547 \tiny $\pm$ .051 & - & \textbf{.608} \tiny $\pm$ .038 & .566 \tiny $\pm$ .055 \\
mushroom & .971 \tiny $\pm$ .001 & .576 \tiny $\pm$ .062 & \textbf{.770} \tiny $\pm$ .044 & .546 \tiny $\pm$ .083 & \textbf{.716} \tiny $\pm$ .052 & .694 \tiny $\pm$ .059 \\
nursery & 1.000 \tiny $\pm$ .000 & .535 \tiny $\pm$ .024 & \textbf{.731} \tiny $\pm$ .047 & .620 \tiny $\pm$ .078 & .654 \tiny $\pm$ .008 & \textbf{.685} \tiny $\pm$ .050 \\
vote & .944 \tiny $\pm$ .013 & \textbf{.721} \tiny $\pm$ .065 & .624 \tiny $\pm$ .068 & .496 \tiny $\pm$ .119 & .612 \tiny $\pm$ .035 & \textbf{.689} \tiny $\pm$ .082 \\ \bottomrule
\end{tabular}
\end{table*}
\endgroup

\begingroup
\renewcommand{\arraystretch}{1.11}

\begin{table*}[tb]
\setlength{\tabcolsep}{3.5pt}
\centering
\caption{5-fold cross-validated mean test accuracy scores and standard errors at $\epsilon {=} 1$ for trees of depth 4. PrivaTree* uses non-private quantiles, DPGDF only ran on categorical features.}
\label{tab:performance-comparison-1}
\begin{tabular}{l|c|cc|ccc}
\toprule
\textbf{OpenML dataset} & \textbf{decision tree} & \textbf{BDPT} & \textbf{PrivaTree*} & \textbf{DPGDF} & \textbf{DiffPrivLib} & \textbf{PrivaTree} \\ 
    & no privacy & \multicolumn{2}{c|}{leaking numerical splits} & \multicolumn{3}{c}{differential privacy} \\ \midrule \multicolumn{7}{c}{Numerical data} \\ \midrule
Bioresponse & .711 \tiny $\pm$ .006 & .505 \tiny $\pm$ .005 & \textbf{.557} \tiny $\pm$ .006 & - & .518 \tiny $\pm$ .007 & \textbf{.576} \tiny $\pm$ .024 \\
Diabetes130US & .606 \tiny $\pm$ .001 & .544 \tiny $\pm$ .002 & \textbf{.599} \tiny $\pm$ .002 & - & .531 \tiny $\pm$ .007 & \textbf{.559} \tiny $\pm$ .001 \\
Higgs & .657 \tiny $\pm$ .001 & timeout & \textbf{.659} \tiny $\pm$ .000 & - & .504 \tiny $\pm$ .002 & \textbf{.601} \tiny $\pm$ .002 \\
MagicTelescope & .781 \tiny $\pm$ .006 & .500 \tiny $\pm$ .000 & \textbf{.753} \tiny $\pm$ .004 & - & .665 \tiny $\pm$ .039 & \textbf{.755} \tiny $\pm$ .006 \\
MiniBooNE & .871 \tiny $\pm$ .001 & .601 \tiny $\pm$ .004 & \textbf{.863} \tiny $\pm$ .002 & - & .505 \tiny $\pm$ .003 & \textbf{.765}\tiny $\pm$ .012 \\
bank-marketing & .771 \tiny $\pm$ .005 & .599 \tiny $\pm$ .004 & \textbf{.745} \tiny $\pm$ .002 & - & .523 \tiny $\pm$ .009 & \textbf{.742} \tiny $\pm$ .003 \\
california & .783 \tiny $\pm$ .002 & .500 \tiny $\pm$ .000 & \textbf{.765} \tiny $\pm$ .002 & - & .547 \tiny $\pm$ .011 & \textbf{.758} \tiny $\pm$ .004 \\
covertype & .740 \tiny $\pm$ .001 & .529 \tiny $\pm$ .001 & \textbf{.745} \tiny $\pm$ .003 & - & .527 \tiny $\pm$ .006 & \textbf{.729} \tiny $\pm$ .001 \\
credit & .748 \tiny $\pm$ .001 & .512 \tiny $\pm$ .012 & \textbf{.743} \tiny $\pm$ .005 & - & .513 \tiny $\pm$ .005 & \textbf{.581} \tiny $\pm$ .007 \\
default-of-credit. & .704 \tiny $\pm$ .002 & .526 \tiny $\pm$ .016 & \textbf{.685} \tiny $\pm$ .002 & - & .557 \tiny $\pm$ .021 & \textbf{.688} \tiny $\pm$ .002 \\
electricity & .734 \tiny $\pm$ .002 & .609 \tiny $\pm$ .002 & \textbf{.738} \tiny $\pm$ .003 & - & .532 \tiny $\pm$ .015 & \textbf{.635} \tiny $\pm$ .002 \\
eye\_movements & .574 \tiny $\pm$ .003 & .500 \tiny $\pm$ .000 & \textbf{.533} \tiny $\pm$ .007 & - & \textbf{.511} \tiny $\pm$ .006 & .506 \tiny $\pm$ .005 \\
heloc & .704 \tiny $\pm$ .004 & .650 \tiny $\pm$ .011 & \textbf{.694} \tiny $\pm$ .003 & - & .575 \tiny $\pm$ .027 & \textbf{.695} \tiny $\pm$ .003 \\
house\_16H & .815 \tiny $\pm$ .004 & .708 \tiny $\pm$ .010 & \textbf{.788} \tiny $\pm$ .010 & - & .567 \tiny $\pm$ .019 & \textbf{.708} \tiny $\pm$ .006 \\
jannis & .715 \tiny $\pm$ .002 & .579 \tiny $\pm$ .032 & \textbf{.704} \tiny $\pm$ .002 & - & .531 \tiny $\pm$ .007 & \textbf{.701} \tiny $\pm$ .004 \\
pol & .929 \tiny $\pm$ .003 & .653 \tiny $\pm$ .024 & \textbf{.904} \tiny $\pm$ .004 & - & .572 \tiny $\pm$ .019 & \textbf{.883} \tiny $\pm$ .007 \\ \midrule
\multicolumn{7}{c}{Numerical \& categorical data} \\ \midrule
albert & .640 \tiny $\pm$ .002 & .632 \tiny $\pm$ .002 & \textbf{.634} \tiny $\pm$ .002 & .505 \tiny $\pm$ .002 & .510 \tiny $\pm$ .003 & \textbf{.593} \tiny $\pm$ .005 \\
compas-two-years & .672 \tiny $\pm$ .006 & .633 \tiny $\pm$ .008 & \textbf{.650} \tiny $\pm$ .011 & .584 \tiny $\pm$ .011 & .574 \tiny $\pm$ .004 & \textbf{.606} \tiny $\pm$ .010 \\
covertype & .756 \tiny $\pm$ .000 & .613 \tiny $\pm$ .001 & \textbf{.756} \tiny $\pm$ .001 & .547 \tiny $\pm$ .008 & .512 \tiny $\pm$ .005 & \textbf{.755} \tiny $\pm$ .002 \\
default-of-credit. & .707 \tiny $\pm$ .004 & .500 \tiny $\pm$ .000 & \textbf{.689} \tiny $\pm$ .003 & .528 \tiny $\pm$ .005 & .531 \tiny $\pm$ .014 & \textbf{.692} \tiny $\pm$ .003 \\
electricity & .732 \tiny $\pm$ .002 & .608 \tiny $\pm$ .004 & \textbf{.738} \tiny $\pm$ .003 & .521 \tiny $\pm$ .004 & .573 \tiny $\pm$ .016 & \textbf{.642} \tiny $\pm$ .006 \\
eye\_movements & .579 \tiny $\pm$ .001 & .499 \tiny $\pm$ .002 & \textbf{.507} \tiny $\pm$ .010 & \textbf{.531} \tiny $\pm$ .004 & .530 \tiny $\pm$ .009 & .528 \tiny $\pm$ .009 \\
road-safety & .728 \tiny $\pm$ .001 & .460 \tiny $\pm$ .001 & \textbf{.711} \tiny $\pm$ .002 & .629 \tiny $\pm$ .024 & .510 \tiny $\pm$ .002 & \textbf{.712} \tiny $\pm$ .003 \\ \midrule
\multicolumn{7}{c}{UCI datasets (numerical \& categorical)} \\ \midrule
adult & .840 \tiny $\pm$ .001 & .811 \tiny $\pm$ .003 & \textbf{.820} \tiny $\pm$ .001 & .754 \tiny $\pm$ .001 & .756 \tiny $\pm$ .003 & \textbf{.822} \tiny $\pm$ .002 \\
breast-w & .950 \tiny $\pm$ .007 & .641 \tiny $\pm$ .009 & \textbf{.927} \tiny $\pm$ .011 & - & \textbf{.950} \tiny $\pm$ .008 & .939 \tiny $\pm$ .007 \\
diabetes & .734 \tiny $\pm$ .006 & .646 \tiny $\pm$ .008 & \textbf{.664} \tiny $\pm$ .007 & - & .655 \tiny $\pm$ .002 & \textbf{.681} \tiny $\pm$ .009 \\
mushroom & .971 \tiny $\pm$ .001 & .956 \tiny $\pm$ .004 & \textbf{.959} \tiny $\pm$ .008 & .695 \tiny $\pm$ .024 & .740 \tiny $\pm$ .049 & \textbf{.949} \tiny $\pm$ .005 \\
nursery & 1.000 \tiny $\pm$ .000 & \textbf{1.000} \tiny $\pm$ .000 & \textbf{1.000} \tiny $\pm$ .000 & .664 \tiny $\pm$ .003 & .789 \tiny $\pm$ .057 & \textbf{1.000} \tiny $\pm$ .000 \\
vote & .944 \tiny $\pm$ .013 & .862 \tiny $\pm$ .032 & \textbf{.871} \tiny $\pm$ .032 & \textbf{.875} \tiny $\pm$ .016 & \textbf{.875} \tiny $\pm$ .030 & .867 \tiny $\pm$ .057 \\ \bottomrule
\end{tabular}
\end{table*}
\endgroup

\subsection{Poisoning robustness on categorical data}

In the main text, we showed a comparison between the poisoning robustness guarantees of PrivaTree and private logistic regression on numerical datasets. In Table \ref{tab:poison-performance-appendix} we show results on data with categorical features encoded as integers.

% TODO

\begin{table*}[tb]
\centering
\caption{5-fold cross-validated mean test accuracy and poisoning accuracy guarantee against a percentage of poisoned samples on mixed numerical/categorical datasets. Stronger privacy provides stronger poisoning robustness but comes at the cost of clean dataset accuracy. Since \textit{vote} and \textit{diabetes} do not have enough samples, we do not the compute 0.1\% guarantee.}
\label{tab:poison-performance-appendix}

\begin{tabular}{lllrrrr}
\toprule
\textbf{dataset} & \textbf{method} & \textbf{$\epsilon$} & \textbf{accuracy} & \textbf{0.1\% guarantee} & \textbf{0.5\% guarantee} & \textbf{1\% guarantee} \\ \midrule
\multicolumn{7}{c}{Numerical \& categorical data} \\ \midrule
albert & PrivaTree & .01 & .50 & .32 & .05 & .01 \\
 &  & .1 & .51 & .01 & - & - \\
 & DiffPrivLib LR & .01 & \textbf{.52} & \textbf{.33} & \textbf{.05} & \textbf{.01} \\
 &  & .1 & \textbf{.52} & .01 & - & - \\ \midrule
compas-two-years & PrivaTree & .01 & .50 & \textbf{.49} & \textbf{.42} & \textbf{.34} \\
 &  & .1 & \textbf{.58} & .43 & .09 & .01 \\
 & DiffPrivLib LR & .01 & .47 & .45 & .39 & .32 \\
 &  & .1 & .47 & .35 & .07 & .01 \\ \midrule
covertype & PrivaTree & .01 & \textbf{.74} & \textbf{.03} & - & - \\
 &  & .1 & \textbf{.74} & - & - & - \\
 & DiffPrivLib LR & .01 & .55 & .02 & - & - \\
 &  & .1 & .63 & - & - & - \\ \midrule
default-of-credit-card-clients & PrivaTree & .01 & \textbf{.58} & \textbf{.52} & \textbf{.34} & \textbf{.20} \\
 &  & .1 & .55 & .20 & - & - \\
 & DiffPrivLib LR & .01 & .50 & .45 & .29 & .17 \\
 &  & .1 & .49 & .18 & - & - \\ \midrule
electricity & PrivaTree & .01 & .53 & \textbf{.39} & \textbf{.12} & \textbf{.03} \\
 &  & .1 & \textbf{.61} & .03 & - & - \\
 & DiffPrivLib LR & .01 & .52 & .39 & .11 & .02 \\
 &  & .1 & .57 & .03 & - & - \\ \midrule
eye\_movements & PrivaTree & .01 & .50 & \textbf{.47} & \textbf{.37} & \textbf{.27} \\
 &  & .1 & \textbf{.51} & .28 & .03 & - \\
 & DiffPrivLib LR & .01 & .50 & \textbf{.47} & \textbf{.37} & \textbf{.27} \\
 &  & .1 & \textbf{.51} & .28 & .03 & - \\ \midrule
road-safety & PrivaTree & .01 & .54 & \textbf{.22} & \textbf{.01} & - \\
 &  & .1 & \textbf{.69} & - & - & - \\
 & DiffPrivLib LR & .01 & .51 & .21 & .01 & - \\
 &  & .1 & .56 & - & - & - \\ \midrule
\multicolumn{7}{c}{UCI datasets (numerical \& categorical)} \\ \midrule
adult & PrivaTree & .01 & .75 & \textbf{.52} & \textbf{.12} & \textbf{.02} \\
 &  & .1 & \textbf{.79} & .02 & - & - \\
 & DiffPrivLib LR & .01 & .54 & .38 & .09 & .02 \\
 &  & .1 & .76 & .02 & - & - \\ \midrule
breast-w & PrivaTree & .01 & .67 & - & .66 & \textbf{.64} \\
 &  & .1 & \textbf{.87} & - & \textbf{.72} & .53 \\
 & DiffPrivLib LR & .01 & .42 & - & .41 & .40 \\
 &  & .1 & .78 & - & .64 & .47 \\ \midrule
diabetes & PrivaTree & .01 & .55 & - & .53 & .52 \\
 &  & .1 & \textbf{.64} & - & .48 & .35 \\
 & DiffPrivLib LR & .01 & .58 & - & \textbf{.56} & \textbf{.55} \\
 &  & .1 & .42 & - & .31 & .23 \\ \midrule
mushroom & PrivaTree & .01 & .72 & \textbf{.69} & \textbf{.58} & \textbf{.46} \\
 &  & .1 & \textbf{.78} & .52 & .09 & .01 \\
 & DiffPrivLib LR & .01 & .49 & .47 & .39 & .31 \\
 &  & .1 & .60 & .40 & .07 & .01 \\ \midrule
nursery & PrivaTree & .01 & .71 & \textbf{.64} & \textbf{.42} & \textbf{.25} \\
 &  & .1 & \textbf{1.00} & .37 & .01 & - \\
 & DiffPrivLib LR & .01 & .55 & .50 & .33 & .20 \\
 &  & .1 & .95 & .35 & .01 & - \\ \midrule
vote & PrivaTree & .01 & .57 & - & - & .57 \\
 &  & .1 & .57 & - & - & .52 \\
 & DiffPrivLib LR & .01 & \textbf{.60} & - & - & \textbf{.59} \\
 &  & .1 & .46 & - & - & .42 \\ \bottomrule
\end{tabular}
\end{table*}